%% file: cvpr.tex
\documentclass[10pt,twocolumn,letterpaper]{article}

\usepackage{cvpr}              

\usepackage{graphicx}
\usepackage{amsmath}
\usepackage{amsthm}
\usepackage{amssymb}
\usepackage{booktabs}
\usepackage{algorithm2e}
\usepackage{subcaption}
\usepackage{multirow}
\usepackage{multicol}
\usepackage[dvipsnames]{xcolor}
\RestyleAlgo{ruled}

\DeclareMathOperator*{\argmin}{arg\,min}
\definecolor{kevincolor}{RGB}{147,112,219}

\usepackage[pagebackref,breaklinks,colorlinks]{hyperref}

\newtheorem{lemma}{Lemma}
\newtheorem{theorem}{Theorem}

\newtheorem{assumption}{Assumption}
\newtheorem{definition}{Definition}
\usepackage[capitalize]{cleveref}
\crefname{section}{Sec.}{Secs.}
\Crefname{section}{Section}{Sections}
\Crefname{table}{Table}{Tables}
\crefname{table}{Tab.}{Tabs.}

\def\cvprPaperID{4761} 
\def\confName{CVPR}
\def\confYear{2023}

\begin{document}

\title{Trainable Projected Gradient Method for Robust Fine-tuning}


\author{
\textbf{Junjiao Tian\thanks{Work partially done during internship at Meta.}\,\,\textsuperscript{1}
\quad Xiaoliang Dai\textsuperscript{2}
\quad Chih-Yao Ma\textsuperscript{2}} \\ 
\textbf{Zecheng He\textsuperscript{2}
\quad Yen-Cheng Liu\textsuperscript{1}
\quad Zsolt Kira\textsuperscript{1}} \\
\\
\textsuperscript{1}Georgia Institute of Technology
\quad \textsuperscript{2}Meta}

\maketitle

\begin{abstract}

\looseness=-1  
Recent studies on transfer learning have shown that selectively fine-tuning a subset of layers or customizing different learning rates for each layer can greatly improve robustness to out-of-distribution (OOD) data and retain generalization capability in the pre-trained models. However, most of these methods employ manually crafted heuristics or expensive hyper-parameter searches, which prevent them from scaling up to large datasets and neural networks. To solve this problem, we propose Trainable Projected Gradient Method (TPGM) to automatically learn the constraint imposed for each layer for a fine-grained fine-tuning regularization. This is motivated by formulating fine-tuning as a bi-level constrained optimization problem. Specifically, TPGM maintains a set of projection radii, i.e., distance constraints between the fine-tuned model and the pre-trained model, for each layer, and enforces them through weight projections. To learn the constraints, we propose a bi-level optimization to automatically learn the best set of projection radii in an end-to-end manner. {\color{black}Theoretically, we show that the bi-level optimization formulation could explain the regularization capability of TPGM.} Empirically, with little hyper-parameter search cost, TPGM outperforms existing fine-tuning methods in OOD performance while matching the best in-distribution (ID) performance. For example, when fine-tuned on DomainNet-Real and ImageNet, compared to vanilla fine-tuning, TPGM shows $22\%$ and $10\%$ relative OOD improvement respectively on their sketch counterparts.  Code is available at \url{https://github.com/PotatoTian/TPGM}.
\end{abstract}

\input{introdcution}
\input{related_works}
\input{method}
\input{theory}

\input{experiments}

\input{conclusion}

\paragraph{Acknowledgements.} This work was partially supported by Meta, ONR grant N00014-18-1-2829, and GTRI.


{\small
\bibliographystyle{ieee_fullname}
\bibliography{cvpr}
}

\clearpage
\input{appendix}

\end{document}

%% file: introdcution.tex
\section{Introduction}
  
Improving out-of-distribution (OOD) robustness such that a vision model can be trusted reliably across a variety of conditions beyond the in-distribution (ID) training data has been a central research topic in deep learning. For example, domain adaptation~\cite{you2019universal,wang2018deep}, domain generalization~\cite{zhou2022domain,muandet2013domain}, and out-of-distribution calibration~\cite{tian2021geometric} are examples of related fields. More recently, large pre-trained models, such as CLIP~\cite{radford2021learning} (pre-trained on 400M image-text pairs), have demonstrated large gains in OOD robustness, thanks to the ever-increasing amount of pre-training data as well as effective architectures and optimization methods. However, fine-tuning such models to other tasks generally results in worse OOD generalization as the model over-fits to the new data and \textit{forgets} the pre-trained features~\cite{radford2021learning}.  \textit{A natural goal is to preserve the generalization capability acquired by the pre-trained model when fine-tuning it to a downstream task.} 

\input{figure/TPGM_diagram}
A recent empirical study shows that aggressive fine-tuning strategies such as using a large learning rate can decrease OOD robustness~\cite{wortsman2022robust}. We hypothesize that the \textit{forgetting} of the generalization capability of the pre-trained model in the course of fine-tuning is due to \textit{unconstrained} optimization on the new training data~\cite{xuhong2018explicit}. This conjecture is not surprising, because several prior works, even though they did not focus on OOD robustness, have discovered that encouraging a close distance to the pre-trained model weights can improve ID generalization, i.e., avoiding over-fitting to the training data~\cite{xuhong2018explicit,gouk2020distance}. Similarly, if suitable distance constraints are enforced, we expect the model to behave more like the pre-trained model and thus retain more of its generalization capability. The question is \textit{where} to enforce distance constraints and \textit{how} to optimize them? 

Several works have demonstrated the importance of treating each layer differently during fine-tuning. For example, a new work~\cite{lee2022surgical} discovers that selectively fine-tuning a subset of layers can lead to improved robustness to distribution shift. Another work~\cite{shen2021partial} shows that optimizing a different learning rate for each layer is beneficial for few-shot learning. Therefore, we propose to enforce a different constraint for each layer. However, existing works either use manually crafted heuristics or expensive hyper-parameter search, which prevent them from scaling up to large datasets and neural networks. For example, the prior work~\cite{shen2021partial} using evolutionary search for hyper-parameters can only scale up to a custom 6-layer ConvNet and a ResNet-12 for few-shot learning. The computation and time for searching hyper-parameters become increasingly infeasible for larger datasets, let alone scaling up the combinatorial search space to all layers. For example, a ViT-base~\cite{vaswani2017attention} model has 154 trainable parameter groups including both weights, biases, and embeddings\footnote{For example, for a linear  layer $y=\mathbf{W}x+\mathbf{b}$, we need to use separate distance constraints for $\mathbf{W}$ and $\mathbf{b}$.}. This leads to a search space with more than $10^{45}$ combinations even if we allow only two choices per constraint parameter, which makes the search prohibitively expensive.

\looseness=-1 To solve this problem, we propose a trainable projected gradient method (TPGM) to support layer-wise regularization optimization. Specifically, TPGM adopts \textit{trainable} weight projection constraints $\gamma$, which we refer to as \textit{projection radii}, and incorporates them in the forward pass of the main model to optimize. Intuitively, as shown in Fig.~\ref{fig:clip_resnet_constraints}, TPGM maintains a set of weight projection radii $\gamma$ i.e., the distance between the pre-trained model ($\theta_0$) and the current fine-tuned model ($\theta_t$), for each layer of a neural network and updates them. The projection radii control how much "freedom" each layer has to grow. For example, if the model weights increase outside of the norm ball defined by $\gamma$ and $\|\cdot\|$, the projection operator will project them back to be within the constraints. To learn the weight projection radii in a principled manner, we propose to use alternating optimization between the model weights and the projection radii, motivated by formulating fine-tuning as a \textit{bi-level} constrained problem (Sec.~\ref{sec:finetune_min}). {\color{black}We theoretically show that the bi-level formulation could explain the behavior of TPGM (Sec.~\ref{sec:theory}). }

Empirically, we conduct thorough experiments on large-scale datasets, DomainNet~\cite{peng2019moment} and ImageNet~\cite{deng2009imagenet}, using different architectures. Under the premise of preserving ID performance, i.e., OOD robustness should not come at the expense of worse ID accuracy, TPGM outperforms existing approaches with little effort for hyper-parameter tuning. Further analysis of the learned projection radii reveals that lower layers (layers closer to the input) in a network require stronger regularization while higher layers (layers closer to the output) need more flexibility. This observation is in line with the common belief that lower layers learn more general features while higher layers specialize to each dataset~\cite{neyshabur2020being,raghu2019transfusion,yosinski2014transferable,wortsman2022robust}. Therefore, when conducting transfer learning such as fine-tuning, we need to treat each layer differently. Our contributions are summarized below. 

\begin{itemize}
    \item \looseness=-1 We propose a trainable projected gradient method (TPGM) for fine-tuning to automatically learn the distance constraints for each layer in a neural network during fine-tuning.  
    \item We conduct experiments on different datasets and architectures to show significantly improved OOD generalization while matching ID performance. 
    
    \item {\color{black}We theoretically study TPGM using linear models to show that bi-level optimization could explain the regularization capability of TPGM.} 
\end{itemize}

%% file: figure/TPGM_diagram.tex
\begin{figure}
     \centering
     \includegraphics[width=0.45\textwidth]{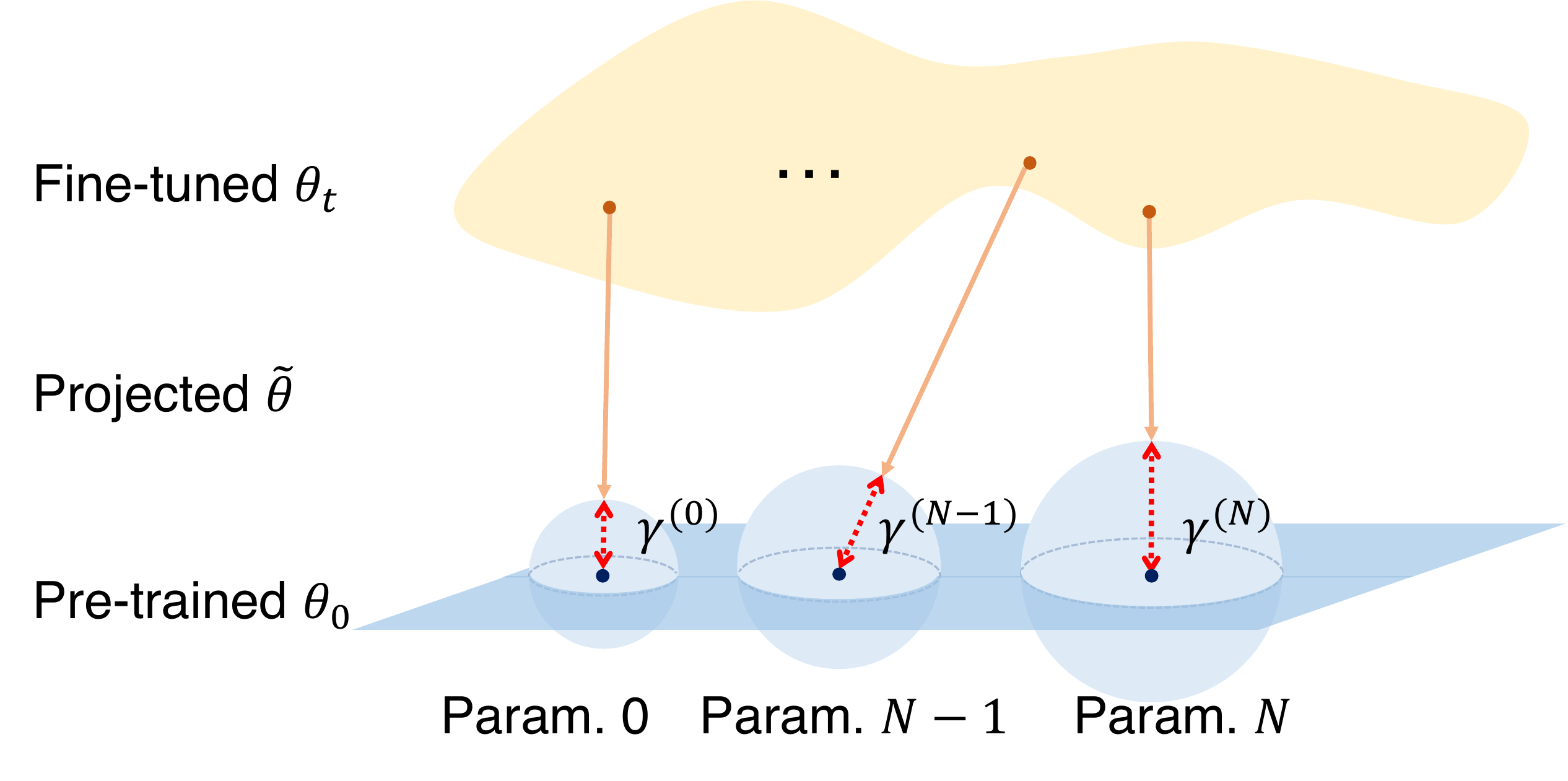}
     \caption{\textbf{Illustration of TPGM.} TPGM learns different weight projection radii, $\gamma$, for each layer between a fine-tuned model $\theta_t$ and a pre-trained model $\theta_0$ and enforces the constraints through projection to obtain a projected model $\Tilde{\theta}$.}
     \label{fig:clip_resnet_constraints}
\end{figure}

%% file: related_works.tex
\section{Related Works}

\textbf{Fine-tuning to Boost ID Performance.}  SpotTune~\cite{guo2019spottune} introduces an additional policy network, which outputs a linear interpolation ratio between the pre-trained model ($\theta_0$) and fine-tuned model ($\theta_t$) based on the input. Instead of directly regularizing the weight space, DELTA~\cite{li2019delta} proposes to regularize the output (feature maps) of $\theta_0$ and $\theta_t$. However, SpotTune introduces an additional network and needs to keep both the pre-trained model and fine-tuned model for inference. DELTA specifically regularizes feature maps generated by convolution layers. In this work, we focus more on general methods, which do not increase inference costs and are applicable to a broader range of models. L2-SP~\cite{xuhong2018explicit} proposes an explicit inductive bias for fine-tuning. Specifically, it uses an L2 regularization to minimize the distance between $\theta_0$ and $\theta_t$. Most recently, MARS-SP~\cite{godwin2021simple} adopts the projected gradient method (PGM) to constrain $\theta_t$ within a small sphere centered on the pre-trained model $\theta_0$. MARS-SP has shown great performance against its prior works. However, we find it very sensitive to hyperparameter tuning. Nonetheless, our work is inspired by and improves PGM by 1) incorporating \textit{trainable} weight projection radii for each layer and 2) developing a \textit{bi-level} optimization algorithm to learn them.   

\textbf{Fine-tuning to Improve OOD Generalization.} As the size of the target dataset increases and better architectures are developed, the benefit from pre-training on the target ID performance diminishes~\cite{he2019rethinking}. However, the power of pre-training goes beyond boosting ID performance. A recent work~\cite{wen2021rethinking} finds that using pre-trained models can greatly improve robustness on OOD datasets and uncertainty-related tasks such as confidence calibration~\cite{guo2017calibration} and OOD detection~\cite{devries2018learning}. Moreover, the fine-tuning strategy used also plays an important role in improving OOD generalization. LP-FT~\cite{kumar2022fine} shows that simultaneously fine-tuning the last linear layer and the feature backbone can distort pre-trained features and thus decreases OOD generalization. A simple strategy of linear probing, i.e., training only the classifier layer, followed by fine-tuning the entire network can greatly mitigate this distortion and improve OOD generalization. WISE~\cite{wortsman2022robust} demonstrates impressive OOD generalization gains by linearly interpolating $\theta_0$ and $\theta_t$. However, this strategy only applies to image-text pre-trained models with zero-shot classifiers such as CLIP~\cite{radford2021learning} because WISE requires the model to have \textit{linear connectivity}. In most cases, linear interpolation between two models results in no better performance than random initialization~\cite{frankle2020linear}.



%% file: method.tex
\section{Method}
In the introduction, we motivated the benefit of explicitly maintaining distance constraints between a pre-trained model and a fine-tuned model~\cite{xuhong2018explicit,gouk2020distance}.  However, it is not clear how to search the space of hyper-parameters (distance constraints) \textit{especially} if we want to do this per layer as the search space grows combinatorially with the number of layers. To this end, we pose the search as a bi-level constrained optimization problem in Sec.~\ref{sec:finetune_min} and introduce closed-form projection in Sec.~\ref{sec:pgm}. Then we present the proposed TPGM algorithm in Sec.~\ref{sec:tpgm}. {\color{black}Finally, we theoretically show that the bi-level optimization design enables TPGM to learn different constraints for each layer in Sec.~\ref{sec:theory}.}

\subsection{Fine-tuning as a Bi-level Constrained Problem}
\label{sec:finetune_min}

Machine learning algorithms usually tune hyper-parameters, e.g., learning rate, weight decay, etc., on a \textit{validation} split. Mathematically, this procedure is equivalent to a bi-level minimization problem. Let $(x,y)$ denote a pair of input data and $\mathcal{L}(\cdot)$ denote the task loss function. The minimization problem can be written as  {\color{black}
\begin{align}
\label{eq:id_robustness}
     \min_{\lambda|(x,y)\in\mathcal{D}_{val}} \mathcal{L}(x,y;\argmin_{\theta_t|(x,y)\in\mathcal{D}_{tr}}  \mathcal{L}(x,y;\theta_t,\lambda),\lambda),
\end{align}}
where $\theta_t$ denotes the trainable model weights and $\lambda$ denotes the hyper-parameters such as learning rate. $\mathcal{D}_{tr}$ is the set of training data and $\mathcal{D}_{val}$ is the set of validation data.

Now, we extend this formulation to fine-tune a pre-trained model. Specifically, it has been shown that maintaining a close \textit{distance} to the pre-trained model improves a model's generalization and robustness~\cite{hendrycks2019using,xuhong2018explicit}. A recent paper~\cite{gouk2020distance} formalizes the concept of maintaining distance as a constrained optimization problem, in which the distance between the new model and the pre-trained model is measured by matrix norms $\|\cdot\|_*$. Mathematically, combined with Eq.~\ref{eq:id_robustness}, we further extend the constrained optimization to a \textit{bi-level constrained}  minimization problem as
\begin{align}
\label{eq:constrained_min}
     \min_{\lambda,\gamma|(x,y)\in\mathcal{D}_{val}} \mathcal{L}(x,y;\argmin_{\theta_t|(x,y)\in\mathcal{D}_{tr}}  \mathcal{L}(x,y;\theta_t,\lambda),\lambda),\\\nonumber \quad\text{s.t.}\quad \|\theta_t-\theta_0\|_*\leq\gamma,
\end{align}
where $\|\theta_t-\theta_0\|_*$ denotes a norm induced distance between the pre-trained model $\theta_0$ and the new model $\theta_t$. Optimizing Eq.~\ref{eq:constrained_min} enforces the model to stay within a distance $\gamma$ from the pre-trained model. 
\subsection{Projected Gradient Method}
\label{sec:pgm}
One method to optimize a constrained problem is \textit{projected gradient method} (PGM)~\cite{iusem2003convergence}. PGM projects the updated model weights to be within the constraint, i.e., $\|\theta_t-\theta_0\|_*\leq\gamma$. However, in general, most projection operations are optimization problems themselves with only a few exceptions having closed from solutions. One example is L2-norm projection $\Pi_{l2}$. Projecting a matrix $\theta_t$ to $\gamma$ distance away, measured by L2 norm, from another matrix $\theta_0$ is a closed form operation as shown in Eq.~\ref{eq:l2_proj}.
\begin{align}
\label{eq:l2_proj}
\Pi_{l2}(\theta_0,\theta_t,\gamma):\Tilde{\theta} = \theta_0 + &\frac{1}{\text{max}\left(1,\frac{\|\theta_t-\theta_0\|_2}{\gamma}\right)}(\theta_t-\theta_0)
\end{align}

The prior work~\cite{gouk2020distance} uses the maximum absolute row sum (MARS) matrix norm because it has a closed form \textit{approximation} that enables fast projection without optimization as well. The MARS approximate projection operator $\Pi_{mars}$ is defined in Eq.~\ref{eq:mars_proj}.
\begin{align}
\label{eq:mars_proj}
\Pi_{mars}(\theta_0,\theta_t,\gamma):\Tilde{\theta} = \theta_0 + &\frac{1}{\text{max}\left(1,\frac{\|\theta_t-\theta_0\|_\infty}{\gamma}\right)}(\theta_t-\theta_0)
\end{align}
$\|\cdot\|_\infty$ denotes the MARS matrix norm, $\|A\|_\infty = \max_j\sum_{i=1}|A_{j,i}|$. Even though we use closed-form projection to avoid additional computation, the projection radius $\gamma$ needs to be pre-determined. Searching for a single weight projection parameter for all layers is already challenging because the scale of $\gamma$ is unknown let alone tailoring a weight projection radius to each layer. In this paper, we do not investigate specific properties of projections, which are orthogonal to our contributions. Therefore, we will benchmark both projections and report the one with better results and leave the comparison to Appendix.

\subsection{Trainable Projected Gradient Method (TPGM)}
Inspired by the projected gradient method, we propose a \textit{trainable} approach to solve the bi-level constrained problem in Eq.~\ref{eq:constrained_min} by integrating the projection operator in Eq.~\ref{eq:l2_proj} or Eq.~\ref{eq:mars_proj} into the forward pass of a model, through which the weight projection radii $\gamma$ can be learned through backpropagation. Specifically, the algorithm consists of three functions: \textit{model update}, \textit{projection update}, and \textit{projection}. A summary of TPGM is presented in Alg.~\ref{alg:alg1} and details of the \textit{projection update} function are in Alg.~\ref{alg:alg2}.

\label{sec:tpgm}
\input{algorithm}

\textbf{Model Update.} TPGM first takes an \textit{unconstrained} gradient descent step. Let $\theta_{t+1}$ denote the updated model parameters at the gradient descent step $t$ for $t\geq0$ where $\theta_0$ denotes the pre-trained initialization. This update is calculated on the loss function $\mathcal{L}(x,y;\theta_t)$ where $(x,y)$ are sampled \textit{training} data, i.e., $(x,y)\in\mathcal{D}_{tr}$. This corresponds to a regular gradient descent step in the conventional setting and the inner minimization in Eq.~\ref{eq:constrained_min}. For example, if vanilla SGD is used in this step, then,
\begin{align*}
    \theta_{t+1} = \theta_t-\eta_t\nabla_\theta\mathcal{L}(x,y;\theta_t), \quad (x,y)\in \mathcal{D}_{tr}.
\end{align*}
Any other optimizers, e.g., Adam~\cite{kingma2014adam}, can be used instead.

\textbf{Projection Update.} The \textit{projection update} function optimizes the weight projection parameters $\gamma_t$ iteratively. As shown in Alg.~\ref{alg:alg2}, the optimization loops for $T_{proj}$ steps. In Alg.~\ref{alg:alg2}, we use SGD as an example for clarity. Any other optimizer can be used. Specifically, using $\theta_{t+1}$ and the closed form projection operation in Eq.~\ref{eq:l2_proj} (or Eq.~\ref{eq:mars_proj}), we construct a projected model $\Tilde{\theta}$ with a \textit{trainable} projection parameter $\gamma_t$ for $t\geq0$ for each layer, where $\gamma_0$ is initialized to a small value $\epsilon$. Then, we optimize the projection parameters using the loss function $\mathcal{L}(x,y;\gamma_t)$ where $(x,y)$ are sampled \textit{validation} data, i.e., $(x,y)\in\mathcal{D}_{val}$. Crucially, in this step, only the weight projection parameters $\gamma_t$ are updated while the updated model $\theta_{t+1}$ remains \textit{frozen}. In other words, no gradients of the model are calculated on the validation data. This is important to avoid contamination of the validation data. The projection update function corresponds to the outer minimization of the constrained problem in Eq.~\ref{eq:constrained_min}.

\textbf{Projection.} Finally, after a new set of projection parameters $\gamma_t$ is updated, we apply the learned projection radii to the updated model $\theta_{t+1}$ using Eq.~\ref{eq:l2_proj} (or Eq.~\ref{eq:mars_proj}). This amounts to enforcing the constraint $\|\theta-\theta_0\|_*\leq\gamma$ in Eq~\ref{eq:constrained_min}. The \textit{projection update} and \textit{projection} functions can be called frequently controlled by a hyperparameter ($f_{proj}$ in Alg.~\ref{alg:alg1}). We will show that for certain pre-trained models, it is sufficient to only call these two functions once at the end of the training, i.e., when $f_{proj}=T-1$ (Sec.~\ref{sec:transformer_exp}). Moreover, we found that in a few cases, TPGM could lead to under-fitting because of its iterative nature. However, the problem can be easily mitigated with total variation smoothing~\cite{chen2010adaptive,condat2013direct}. Since we only observed this in one setting in our experiments, we defer the discussion to Appendix~\ref{sec:smoothing}.

We summarize TPGM in Alg.~\ref{alg:alg1}.  Intuitively, TPGM maintains a set of weight projection parameters for each layer of a neural network and updates them. The projection parameters control how much ``freedom'' each layer has to grow.  As we will observe later, when fine-tuning a model, layers close to the input generally require smaller changes than layers close to the output. This property helps preserve generalization capabilities obtained by the pre-trained model. TPGM inevitably adds some computation overhead. We provide additional discussion on it in Appendix~\ref{sec:computation}.

%% file: algorithm.tex
\begin{algorithm}[t]
\caption{TPGM}\label{alg:alg1}
\KwData {$\mathcal{D}_{tr}$,$\mathcal{D}_{val}$}
\KwResult {$\Tilde{\theta}_{t+1}$}
Initialize $\Tilde{\theta}_0=\theta_0,\gamma_0=\epsilon$ \\
\For{$t=\{0,...,T-1\}$}{
    $\theta_{t+1} =\theta_{t} - \eta_t\nabla_\theta \mathcal{L}(x,y;\Tilde{\theta_{t}})\quad x,y\in\mathcal{D}_{tr}$\\
    \If{$t$ \text{mod} $f_{proj}$  $== 0$}{
    $\gamma_{t+1} =\text{ProjectionUpdate}(\mathcal{D}_{val},\theta_0,\theta_{t+1},\gamma_t)$\\
    $\Tilde{\theta}_{t+1} =\Pi(\theta_0,\theta_{t+1},\gamma_{t+1})$}
}
\end{algorithm}

\begin{algorithm}[t]

\caption{ProjectionUpdate}\label{alg:alg2}
\KwData {$\mathcal{D}_{val}$}
\KwResult {$\gamma_{t+1}$}
\For{$\tau=\{0,...,T_{proj}-1\}$}{
    $\Tilde{\theta} = \Pi(\theta_0,\theta_{t+1},\gamma_\tau)$\\
    $\gamma_{\tau+1} = \gamma_\tau - \zeta\nabla_\gamma \mathcal{L}(x,y;\Tilde{\theta}) \quad x,y\in\mathcal{D}_{val}$
}

\end{algorithm}

%% file: theory.tex
\subsection{Bi-level Optimization}
\label{sec:theory}
 {\color{black}Following a common strategy in studying transfer learning~\cite{wu2020understanding,tripuraneni2020theory,du2020few,xie2020n,wortsman2022robust}, we theoretically study TPGM using linear models and explain why optimizing the bi-level problem in Eq.~\ref{eq:constrained_min} could enable the regularization capability of TPGM.}

\textbf{Problem Setup.}  Let $x\in \mathbb{R}^{d}$ denote an ID data and the corresponding label $y$ is generated by a \textit{ground truth} linear model $\theta_*\in \mathbb{R}^d $, i.e., $y=\theta_*^Tx$.  To construct the training set, we sample $n$ training data, where $n<d$, and stack the sampled data into a data matrix $\mathbf{X}_{tr}\in\mathbb{R}^{d\times n}$. Accordingly, the labels form a vector $\mathbf{Y}_{tr}=\mathbf{X}_{tr}^T\theta_*\in\mathbb{R}^{n}$. The 
\textit{training} goal is to minimize the \textit{empirical} loss.
\begin{align}
    \label{eq:training_obj}
\mathcal{L}(\mathbf{X}_{tr},\mathbf{Y}_{tr};\theta)=\|\mathbf{X}_{tr}^T\theta-\mathbf{Y}_{tr}\|_2
\end{align}

Note that this forms an \textit{over-parameterized} linear system, i.e., there are more parameters than equations, because $n<d$. This is similar to how modern neural networks are over-parameterized with respect to the data.

{\color{black}
\textbf{Complementary Decomposition using SVD.} For the analysis, we make an independence assumption on the data matrix $\mathbf{X}_{tr}$. This assumption exists for notation simplicity and can be relaxed easily.
\begin{assumption}
\label{assmp:independence}
Let the $n$ training data be linearly independent. The following SVD exists for the data matrix $\mathbf{X}_{tr}$.
\begin{align*}
\mathbf{X}_{tr} = \mathbf{U}\mathbf{\Sigma} \mathbf{V}^T,\quad \mathbf{U}\in\mathbb{R}^{d\times n},\mathbf{\Sigma}\in\mathbb{R}^{n\times n}, \mathbf{V}\in\mathbb{R}^{n\times n}.
\end{align*}
\end{assumption}

Consequently, we can decompose any vector $x\in\mathbb{R}^{d}$ into two components, $x=\mathbf{U}\tau + \mathbf{U}_{\bot}\tau_{\bot}$, where $\mathbf{U}$ is the basis for the span of training samples, $\mathbf{U}_{\bot}\in\mathbb{R}^{d\times (d-n)}$ is the basis for the complementary subspace, and $\tau\in\mathbb{R}^{n}$, $\tau_{\bot}\in\mathbb{R}^{d-n}$ are the corresponding coordinates. There are infinitely many solutions to Eq.~\ref{eq:training_obj} because this is an over-parameterized system.  

\begin{definition}
We denote a \textit{projected model} as $\Tilde{\theta} = \theta_0 + \alpha (\theta - \theta_0)$ (obtained using Eq.~\ref{eq:l2_proj} or Eq.~\ref{eq:mars_proj}), where  $\theta$ is one minimizer of Eq.~\ref{eq:training_obj}, $\theta_0$ is the pre-trained model and $0\leq\alpha\leq1$ is the projection ratio. 
\end{definition}

To quantify the effects of projection $\alpha$, we can look at the average performance of the projected model $\Tilde{\theta}$ on test data. Consequently, we investigate the \textit{expected} loss of the projected model over the entire data space. 
\begin{align}
    \label{eq:expected_obj}
    \mathbb{E}[\mathcal{L}(x,y;\Tilde{\theta})]= \mathbb{E}\left[\left\|{\Tilde{\theta}}^Tx-y \right\|_2\right]
\end{align}
We provide the following theorem to shed light on how projection affects the \textit{expected} loss and what it depends on.}

\begin{theorem}
\label{thm:thm_1}
Let Assumption~\ref{assmp:independence} hold, the expected loss of $\Tilde{\theta}$ in Eq.~\ref{eq:expected_obj} is upper bounded as the following,
\begin{align}
\label{eq:expected_upper}
    \mathbb{E}\left[\left\|\Tilde{\theta}^Tx-y \right\|_2\right] & \leq \underbrace{(1-\alpha)\epsilon{\tau}}_{in-span} + \underbrace{\left(\epsilon+\alpha \left\|{\theta}-\theta_0 \right\|_2\right){\tau}_{\bot}}_{out-span},
\end{align}

\noindent where $\epsilon = \left\|{\theta}_0-\theta_* \right\|_2$ and ${\tau} \doteq \mathbb{E}[\|\tau\|_2]$, ${\tau}_\bot \doteq \mathbb{E}[\|\tau_\bot\|_2]$. A complete proof is provided in Appendix~\ref{sec:proof}.
\end{theorem}
\looseness=-1  The upper bound in Thm.~\ref{thm:thm_1} has two components, risk due to components in the training data span (in-span) and risk due to components in the complementary subspace (out-span). To minimize the expected loss, one will expect $\alpha$ to be dependent on the value of $\epsilon$. Recall that the quantity $\epsilon$ is the distance between the pre-trained model and the ground truth model and can be viewed as a measure of how ``good" the pre-trained model is. {\color{black}Therefore, we expect two types of behaviors from $\alpha$ depending on $\epsilon$: 
\begin{itemize}
    \item \textbf{When $\epsilon$ is small, $\alpha$ needs to be smaller} to minimize the second component, meaning stronger projection.
    \item \textbf{When $\epsilon$ is large}, \textbf{$\alpha$ needs to be larger} to minimize the first component, meaning weaker projection.
\end{itemize}
}

\looseness=-1 {\color{black} The theorem indicates that if we optimize the projected model $\tilde{\theta}$ on a \textit{separate} batch of data, different from the data the model gradients are calculated on, the projection ratio $\alpha$ will seek to balance between fitting the training data (in-span) and generalizing to new data (out-span). For example, when $\epsilon$ is small, i.e., the pre-trained model is close to the optimal model, the formulation encourages stronger projection. 
} 

Furthermore, as prior works have found that lower layers tend to learn more general features while higher layers specialize to a specific dataset, $\epsilon$ is likely to be smaller for the lower layers and larger for higher layers because pre-trained models likely have learned very good low-level general features~\cite{neyshabur2020being,raghu2019transfusion,yosinski2014transferable,wortsman2022robust}. This offers one explanation of why TPGM automatically learns different constraints for each layer. Therefore, we hypothesize that optimizing the projection radii on a dataset sampled \textit{separately} from the training data, e.g., the validation dataset, is essential to learning different constraints for each layer.

%% file: experiments.tex
\section{Experiments}
\textbf{Overview.} To validate TPGM, we conduct experiments on large-scale datasets using different architectures. The experiments are split into two sections depending on the specific architecture used. In Sec.~\ref{sec:resnet_exp}, we use ResNet~\cite{he2016deep} with a \textit{CLIP pre-trained ResNet50}~\cite{radford2021learning} and an \textit{ImageNet pre-trained MOCO-V3 ResNet50\cite{chen2021empirical}} as the pre-trained models. In Sec.~\ref{sec:transformer_exp}, we use Vision Transformers~\cite{vaswani2017attention} with a \textit{CLIP pre-trained ViT-B model}~\cite{radford2021learning}.


\textbf{Datasets.} For the ResNet experiments, we use DomainNet~\cite{peng2019moment} (0.6M images over 345 classes) as the benchmark. DomainNet has five domains: real, sketch, painting, inforgraph, and clipart. We use the real domain as the ID fine-tuning domain (with held-out test data to test ID performance) and the rest as the OOD domains. For the Transformer experiments, we use ImageNet-1K\cite{deng2009imagenet} as the fine-tuning dataset. For the ID test dataset, we add ImageNet-V2~\cite{recht2019imagenet} in addition to ImageNet-1K.  For the OOD test datasets, we use ImageNet-A~\cite{hendrycks2021natural}, ImageNet-R~\cite{hendrycks2021many}, and ImageNet-S~\cite{wang2019learning}. No OOD data are used during training.

\subsection{Fine-Tuning a Pre-trained ResNet}
\label{sec:resnet_exp}
\input{table/clip_resnet_main}
\input{table/moco_resnet_main}
\input{table/clip_resnet_main_10}

In this section, we compare TPGM to several existing methods using a CLIP pre-trained ResNet50~\cite{radford2021learning} and ImageNet pre-trained MOCO-V3 ResNet50~\cite{chen2021empirical} as initialization. Specifically for TPGM, we use $f_{proj} = 1$, $T_{proj} =1$, meaning that \textit{projection update} and \textit{projection} are activated at every gradient descent step (Alg.~\ref{alg:alg1}). We also use the MARS projection in Eq.~\ref{eq:mars_proj} because we found that MARS projection performs better than L2 projection in this setting (Appendix~\ref{sec:compare_proj}). Moreover, we do not include WISE~\cite{wortsman2022robust} in this comparison because we found that CLIP pre-trained ResNet50 has poor linear connectivity, i.e., linear interpolation results in drastic degradation of performance (Appendix~\ref{sec:resnet_wise}). Therefore, we do not use any zero-shot classifiers for initializing the last linear layer (See sec.~\ref{sec:transformer_exp} for a detailed description of zero-shot classifiers). The recipe for training ResNet is relatively simple. We use the Adam optimizer~\cite{kingma2014adam} with default settings and a batch size of 256. Models are fine-tuned for 50 epochs with a cosine learning rate schedule. The same training recipe is used for all experiments unless otherwise specified. Implementation details are provided in Appendix~\ref{sec:implementation}.

\textbf{TPGM improves OOD robustness without sacrificing ID performance.} We first present the main results on DomainNet using CLIP pre-trained ResNe50. As shown in Tab.~\ref{tab:clip_resnet}, we observe that both L2-SP and LP-FT bring significant improvements to OOD generalization with respect to vanilla FT while matching or surpassing its ID accuracy.  Nevertheless, TPGM brings the most OOD improvement while also surpassing vanilla FT on ID accuracy. We also report results using MOCO-V3 in Tab.~\ref{tab:moco_resnet}. MOCO-V3 is pre-trained on ImageNet-1K (1.2M) consisting of mainly real images, a much smaller and less diverse pre-training data set than CLIP's. Therefore, we see worse OOD generalization results from all methods, compared to using CLIP pre-trained models (Tab.~\ref{tab:clip_resnet}). This indicates that the size and diversity of the pre-training dataset have a huge impact on generalization. Nevertheless, TPGM~\footnote{ TPGM on MocoV3 is the only situation where we found total variation smoothing (see Appendix~\ref{sec:smoothing}) helps with ID performance. Without smoothing, TPGM achieves 81.66 ID and 37.27 Ave. OOD performance.} yields the best OOD performance while matching the best ID performance.


\input{figure/resnet_analysis}
\textbf{TPGM adjusts to the size of training data.} As an \textit{automatic} regularization method, TPGM also needs to adjust to different regularization strengths according to the size of the training set. TPGM can avoid over-fitting to a small fine-tuning dataset through the outer minimization loop of the projection parameters on validation data (Eq.~\ref{eq:constrained_min}). In this section, we additionally present results when we reduce the DomainNet-Real data to 10$\%$ of its original size. We follow a similar strategy as in the $100\%$ experiments and sweep different learning rates (for competing methods we sweep their hyperparameters). All models are trained for 150 epochs. In Tab.~\ref{tab:clip_resnet_10}, we observe significant degradation in ID performance across all methods except for PF and TPGM. PF only trains the Batch-norm layers and therefore is less prone to over-fitting. TPGM achieves an even higher ID performance because it learns small projection radii, which project the fine-tuned model closer to the pre-trained model. To see it, we visualize the average distance between the fined-tuned model and the pre-trained model for each residual block using TPGM for both 100$\%$ and $10\%$ data in Fig.~\ref{fig:per_layer_para_resnet}. We observe that 1) lower layers have smaller constraints and higher layers have larger constraints, meaning more freedom to grow, and 2) with only $10\%$ training data, the learned constraints are much smaller than those trained with 100$\%$. 
This behavior explains why TPGM maintains high ID performance and avoids over-fitting with fewer training data because it chooses to rely more on the pre-trained model by enforcing stronger projection.

\subsection{Fine-tuning a Pre-Trained Transformer}
\label{sec:transformer_exp}
 In this section, we compare to existing fine-tuning methods using a CLIP pre-trained ViT-B model. We initialize all models with the pre-trained weights and also the last linear classifier layer with a \textit{zero-shot classifier} extracted from the CLIP text-encoder. Specifically, for TPGM, we use $f_{proj} = T-1$ and $T_{proj} =200$, meaning that projection only happens \textit{once} at the end of fine-tuning. This is possible because CLIP pre-trained ViT-B has been shown to have good linear connectivity~\cite{wortsman2022robust} in contrast to the CLIP pre-trained ResNet (Appendix~\ref{sec:resnet_wise}). Furthermore, we use L2 projection in Eq.~\ref{eq:l2_proj} because we found L2 projection is better than MARS projection in this setting (Appendix~\ref{sec:compare_proj}). Training Transformers~\cite{vaswani2017attention} requires careful tuning of the training recipe to achieve the best results\footnote{Our training recipe yields 84.20 vanilla FT accuracy on ImageNet using a CLIP ViT-B, which is significantly better than prior works, e.g., WISE~\cite{wortsman2022robust} reported 81.3, FT-LP~\cite{radford2021learning} reported 81.7 on the same dataset.}. We follow some common practices in prior works~\cite{touvron2021training,touvron2022deit} to boost performance. We leave implementation details in Appendix~\ref{sec:implementation}.  

\textbf{Extracting a Zero-Shot Classifier.} CLIP has an image-encoder $g(\cdot)$ and a text-encoder $h(\cdot)$, and is capable of zero-shot classification. For example, given an image $x$ and its label space $y\in\mathcal{Y}=\{y_1,...,y_c\}$, zero-shot classification can be done by first inserting the class name $y_i$, e.g., "apple", into a template $c_i$, e.g., "a photo of $\{$apple$\}$" and extracting its text embedding $h(c_i)$, and then computing an inner product,$\langle h(c_i), g(x)\rangle$, between the text embedding and the corresponding image embedding. The maximum value of the inner product over all classes determines the membership of the input. Following the prior work~\cite{wortsman2022robust}, one can stack $h(c_i),\quad \forall i\in\{1,...,c\}$ into a weight matrix $\mathbf{W}_{\text{zero-shot}}$ as a zero-shot classification layer. We use this weight matrix as initialization as well as zero-shot classification.



\input{table/imagenet_main}
\textbf{TPGM Improves OOD robustness without sacrificing ID performance.} Now, we present the main benchmark results, accuracy on each of the datasets, and percentage of improvement with respect to the vanilla FT method, in Tab.~\ref{tab:imagenet}. Parameter-efficient methods such as LP and BitFit all improve OOD generalization however at a loss of ID performance. We hypothesize that they help preserve generalization by updating fewer parameters in the network, and therefore maintaining a closer distance to the pre-trained model. On the other hand, the restriction on the function space can result in under-fitting, manifested in lower ID performance. Surprisingly, L2-SP and LP-FT fail to improve either ID or OOD performance. We think this is because the added regularization in L2-SP and the two-stage training procedure in LP-FT are not very compatible with the existing Transformer training recipe. The zero-shot classifier brings significant OOD improvement even though the ID performance is way worse than the FT model. This confirms that CLIP models acquire great generalization capability during pre-training, as also reported by the original paper~\cite{radford2021learning}. TPGM and WISE perform notably better than other methods. We will elaborate on the comparison next. 

\textbf{TPGM outperforms WISE.} The current state-of-the-art method for fine-tuning a pre-trained model with linear connectivity is WISE~\cite{wortsman2022robust}, which linearly interpolates between a fine-tuned model and the pre-trained model with a \textit{single} ratio. For lack of a better heuristic, the paper suggests 0.5 as the interpolation ratio and leaves the research for a better method to determine the mixing ratio as an open question. The comparison between TPGM and WISE comes down to the comparison between optimized per-layer constraints and a hand-tuned single constraint. Therefore, for WISE, we sweep different ratios from 0.1 to 0.9, controlling the distance to the pre-trained model from close to far. For TPGM, to fairly compare to WISE, we put an L2 regularization on the magnitude of the trainable projection parameters with a hyperparameter $\mu$ that controls the strength of regularization. Intuitively, a larger regularization forces the projection radii to be smaller, meaning projecting the fine-tuned model closer to the pre-trained model. We sweep a range of different $\mu$ from $4e^{-3}$ to $0.0$. We refer to this variant as TPGM-C (C for \textbf{c}ontrolled). Note that this L2 regularization is not a hyper-parameter in the algorithm itself.  In Fig.~\ref{fig:tpgm_wise}, we observe a trade-off between the ID performance and the OOD performance for both methods. However, TPGM  clearly outperforms WISE because for the same ID performance, TPGM has better OOD performance and for the same OOD performance, TPGM has better ID performance. This demonstrates the benefits of maintaining per-layer constraints over a single interpolation ratio. We also provide the same experiment and visualization using a CLIP pre-trained ViT-L in Appendix~\ref{sec:vit_l}.
\input{figure/imagenet_analysis}
\input{figure/imagenet_analysis_b}

\textbf{Different layers require different regularization.} Now we take a closer look at the learned TPGM projection radii especially in terms of ``closeness'' to the pre-trained model.  In Fig.~\ref{fig:per_layer_para}, we visualize the average distance from the pre-trained model for each transformer block with three different L2 regularization strengths. We observe that 1) lower layers have smaller projection radii, i.e., they are more tightly constrained whereas higher layers have larger projection radii and therefore more freedom to grow; 2) as the regularization strength on projection radii increases, on average, they become closer to the pre-trained model while still following the previous observation.  Combined with the common belief that lower layers learn more general features and higher layers learn more specialized layers, we hypothesize that lower layers of the pre-trained model are ``closer'' to the ideal model than higher layers. This observation corroborates with our theoretical analysis (Sec.~\ref{sec:theory}) that when the distance between the pre-trained model and the ideal model is small, TPGM favors close projection.

%% file: table/clip_resnet_main.tex
\begin{table*}[t]
\caption{\textbf{DomainNet Results using CLIP pre-trained ResNet50 with 100$\%$ Real Data.} TPGM improves OOD performance significantly even when a zero-shot classifier is not available.
}
\centering

\resizebox{0.85\linewidth}{!}{
\begin{tabular}{c|c|cccc|ccc} 

\toprule
&\multicolumn{1}{c|}{ID} & \multicolumn{4}{c|}{OOD} & \multicolumn{3}{c}{Statistics} \\

&Real & Sketch & Painting & Infograph & Clipart &  OOD Avg. & ID $\Delta$ ($\%$) & OOD $\Delta$ ($\%$)\\
\midrule
Vanilla FT & 80.93 \color{Gray}(0.08) & 31.81 \color{Gray}(0.06) & 41.02 (0.10)  & 20.29 \color{Gray}(0.08) & 43.59 \color{Gray}(0.15) & 34.18 & 0.00 & 0.00 \\
LP &52.56 \color{Gray}(0.09)&	20.05 \color{Gray}(0.21)&	24.92 \color{Gray}(2.49)&	19.18 \color{Gray}(0.46)&	21.15 \color{Gray}(0.18)&	21.33&	\color{red}-35.05&	\color{red}-37.60\\
PF~\cite{kanavati2021partial} &78.27 \color{Gray}(0.11) &	36.77 \color{Gray}(0.32)&	42.13 \color{Gray}(0.35)&	24.71 \color{Gray}(0.18)&	43.31 \color{Gray}(0.53)&	36.73&	\color{red}-3.29&	\color{Green}7.46\\
L2-SP~\cite{xuhong2018explicit}& 82.07 \color{Gray}(0.09)&	36.67 \color{Gray}(0.11)&	\textbf{45.62} \color{Gray}(0.35)&	22.97 \color{Gray}(0.42)&	47.78 \color{Gray}(0.30)&	38.26&	\color{Green}1.40&	\color{Green}11.94\\
MARS-SP~\cite{gouk2020distance}& 77.19 \color{Gray}(0.63)&	25.33 \color{Gray}(1.07)&	33.43 \color{Gray}(2.06)&	14.81 \color{Gray}(0.43)&	39.20 \color{Gray}(0.74)&	28.19&	\color{red}-4.62&	\color{red}-17.53\\
LP-FT~\cite{kumar2022fine}& 80.82 \color{Gray}(0.95)& 	34.85 \color{Gray}(1.93)& 	44.03 \color{Gray}(0.05)& 	22.23 \color{Gray}(2.01)& 	46.13 \color{Gray}(2.34)& 36.81& 		\color{red}-0.14& 	\color{Green}7.69\\
\midrule
TPGM &\textbf{83.64} \color{Gray}(0.01)&	\textbf{38.78} \color{Gray}(0.42)&	43.11 \color{Gray}(0.25)&	\textbf{28.70} \color{Gray}(0.31)&	\textbf{48.01} \color{Gray}(0.25)&	\textbf{39.65}&	\color{Green}\textbf{3.34}&	\color{Green}\textbf{16.01}\\
\bottomrule
\end{tabular}
}
\label{tab:clip_resnet}
\end{table*}

%% file: table/moco_resnet_main.tex
\begin{table*}[t]
\caption{\textbf{DomainNet Results using MOCO-V3 pre-trained ResNet50 Results with 100$\%$ Real Data.} TPGM improves OOD generalization using a self-supervised pre-trained model while improving ID performance. 
}
\centering

\resizebox{0.85\linewidth}{!}{
\begin{tabular}{c|c|cccc|ccc} 

\toprule
&\multicolumn{1}{c|}{ID} & \multicolumn{4}{c|}{OOD} & \multicolumn{3}{c}{Statistics} \\

&Real & Sketch & Painting & Infograph & Clipart &  OOD Avg. & ID $\Delta$ ($\%$) & OOD $\Delta$ ($\%$)\\
\midrule
Vanilla FT & 81.99 \color{Gray}(0.03)&	31.52 \color{Gray}(0.33)&	42.89 \color{Gray}(0.53)&	18.51 \color{Gray}(0.28)&	44.98 \color{Gray}(0.24)&	34.47&	0.00&	0.00 \\
LP & 73.01 \color{Gray}(0.03)&	24.10 \color{Gray}(0.23)&	39.56 \color{Gray}(0.15)&	12.27 \color{Gray}(0.02)&	30.38 (\color{Gray}0.08)&	26.58&	\color{red}-10.96&	\color{red}-22.90\\
PF~\cite{kanavati2021partial} & 78.27 \color{Gray}(0.03)&	27.72 \color{Gray}(0.07)&	39.74 \color{Gray}(0.12)&	15.56 \color{Gray}(0.08)&	38.18 \color{Gray}(0.12)&	30.30&	\color{red}-4.55&	\color{red}-12.11\\
L2-SP~\cite{xuhong2018explicit}& 81.51 \color{Gray}(0.02)&	34.91 \color{Gray}(0.22)&	45.76 \color{Gray}(0.16)&	18.97 \color{Gray}(0.11)&	45.29 \color{Gray}(0.18)&	36.23&	\color{red}-0.59&	\color{Green}5.09 \\
MARS-SP~\cite{gouk2020distance}& 81.89 \color{Gray}(0.01)&	34.44 \color{Gray}(2.54)&	45.05 \color{Gray}(1.91)&	19.97 \color{Gray}(1.48)&	46.36 \color{Gray}(1.29)&	36.45&	\color{red}-0.13&	\color{Green}5.74\\
LP-FT~\cite{kumar2022fine}& \textbf{82.92} \color{Gray}(0.01)& 34.50 \color{Gray}(0.22)&	45.42 \color{Gray}(0.31)&	\textbf{20.12} \color{Gray}(0.43)&	\textbf{47.11} \color{Gray}(0.27)&	36.79&	\color{Green}\textbf{1.13}&	\color{Green}6.72\\
\midrule
TPGM & 82.66 \color{Gray}(0.13)&	\textbf{35.35} \color{Gray}(0.33)&	\textbf{46.20} \color{Gray}(0.20)&	\textbf{20.13 }\color{Gray}(0.12)&	45.75 \color{Gray}(0.12)&	\textbf{36.86}&	\color{Green}0.82&	\color{Green}\textbf{6.91}\\
\bottomrule
\end{tabular}
}
\label{tab:moco_resnet}
\end{table*}

%% file: table/clip_resnet_main_10.tex
\begin{table*}[t]
\caption{\textbf{DomainNet Results using CLIP pre-trained ResNet50 with 10$\%$ Real Data.} TPGM adjusts to the size of the fine-tuning dataset by imposing stronger per-layer constraints.  
}
\centering

\resizebox{0.85\linewidth}{!}{
\begin{tabular}{c|c|cccc|ccc} 

\toprule
&\multicolumn{1}{c|}{ID} & \multicolumn{4}{c|}{OOD} & \multicolumn{3}{c}{Statistics} \\

&Real & Sketch & Painting & Infograph & Clipart &  OOD Avg. & ID $\Delta$ ($\%$) & OOD $\Delta$ ($\%$)\\
\midrule
Vanilla FT & 57.35 \color{Gray}(1.43)&	17.48 \color{Gray}(0.68)&	25.60 \color{Gray}(0.70)&	10.30 \color{Gray}(1.57) &	23.01 \color{Gray}(0.65)&	19.10&	0.00&	0.00 \\

LP & 47.19 \color{Gray}(0.93)&	17.81 \color{Gray}(0.25)&	22.71 \color{Gray}(2.08)&	17.13 \color{Gray}(0.75)&	17.59 \color{Gray}(0.69)&	18.81&	\color{red}-17.71&	\color{red}-1.52\\

PF~\cite{kanavati2021partial} & 71.04 \color{Gray}(0.91)&	27.87 \color{Gray}(1.04)&	\textbf{38.31} \color{Gray}(1.05)&	\textbf{19.85} \color{Gray}(0.70)&	33.92 (\color{Gray}1.53)&	29.99&	\color{Green}23.86&	\color{Green}57.01 \\

L2-SP~\cite{xuhong2018explicit}& 61.41 \color{Gray}(0.92)&	22.61 \color{Gray}(0.52)&	30.48 \color{Gray}(0.42)&	12.28 \color{Gray}(0.50)&	26.59 \color{Gray}(0.57)&	22.99&	\color{Green}7.08&	\color{Green}20.37\\
MARS-SP~\cite{gouk2020distance}& 52.53 \color{Gray}(0.84)&	15.34 \color{Gray}(0.54)&	21.57 \color{Gray}(0.45)&	8.49 \color{Gray}(0.60)&	19.96 \color{Gray}(0.01)&	16.34&	\color{red}-8.41&	\color{red}-14.44 \\

LP-FT~\cite{kumar2022fine}& 64.11 \color{Gray}(0.78)&	20.54 \color{Gray}(0.27)&	30.89 \color{Gray}(0.41)&	13.58 \color{Gray}(0.63)&	29.55 \color{Gray}(0.82)&	23.64&	\color{Green}11.78&	\color{Green}23.77 \\
\midrule
TPGM & \textbf{73.16} \color{Gray}(1.27)&	\textbf{29.88} \color{Gray}(0.81)&	36.80 \color{Gray}(1.42)&	19.72 \color{Gray}(0.12)&	\textbf{35.28} \color{Gray}(0.74)&	\textbf{30.42}&	\color{Green}\textbf{27.56}&	\color{Green}\textbf{59.27}\\
\bottomrule
\end{tabular}
}
\label{tab:clip_resnet_10}
\end{table*}

%% file: figure/resnet_analysis.tex
\begin{figure}
    \centering
 \includegraphics[width=0.4\textwidth]{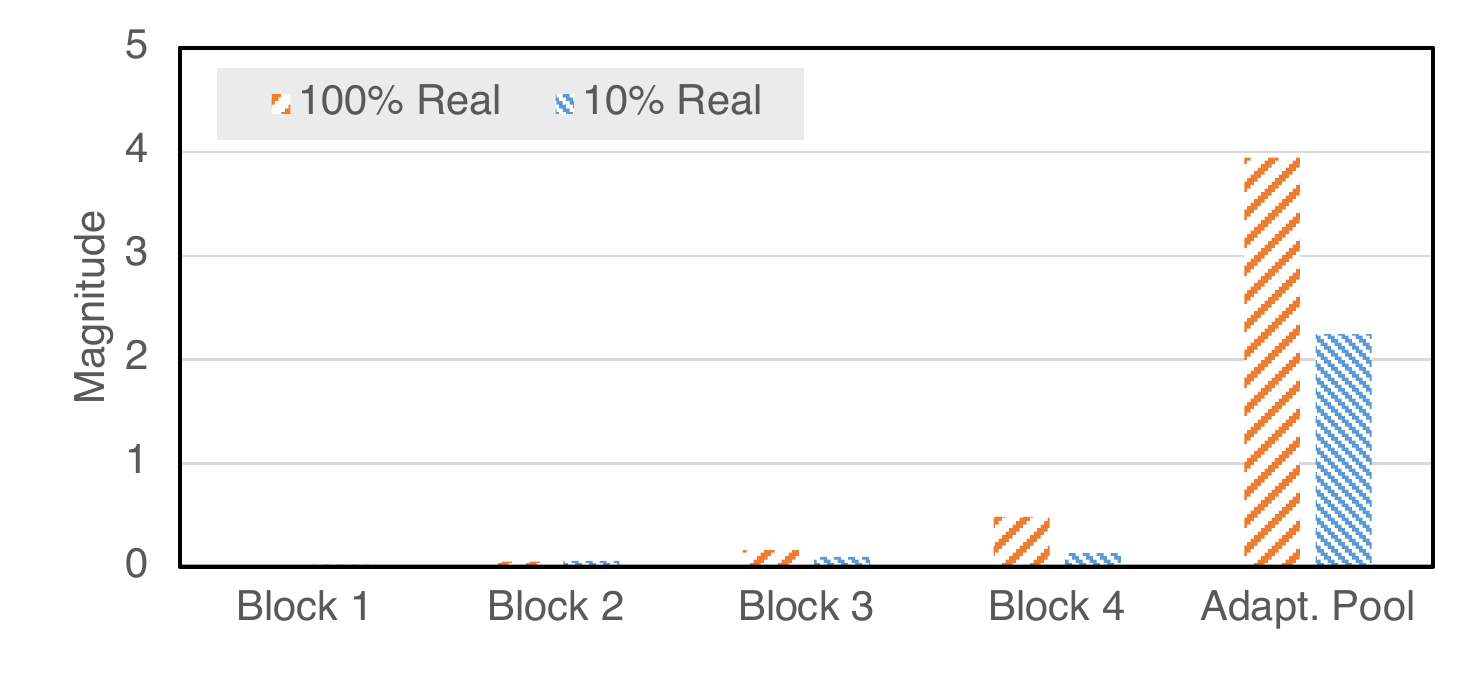}
 \caption{\textbf{Average distance between the fine-tuned model and a CLIP pre-trained ResNet50 for each Residual block using TPGM.} Under the distance constraints imposed by TPGM, most of the model changes are in the last adaptive pooling layer. }
 \label{fig:per_layer_para_resnet}
 \vspace{-3mm}
\end{figure}

%% file: table/imagenet_main.tex
\begin{table*}[t]
\caption{\textbf{ImageNet Results using CLIP pre-trained ViT-B.} TPGM improves OOD performance significantly without losing ID performance. TPGM-C achieves the best OOD performance while maintaining a more competitive ID performance compared to the current state-of-the-art method WISE. TPGM-C is a controlled variant of TPGM, designed to lower its ID performance to the same level as WISE for a fair comparison of OOD performance. Note that prior works~\cite{wortsman2022robust} \textit{sub-sample} classes for ImageNet-A/R (200 classes) for evaluation while we do not. 
}
\centering

\resizebox{0.85\linewidth}{!}{
\begin{tabular}{c|cc|ccc|cccc} 

\toprule
&\multicolumn{2}{c|}{ID} & \multicolumn{3}{c|}{OOD} & \multicolumn{4}{c}{Statistics} \\

&ImageNet & ImageNet-V2 & ImageNet-A & ImageNet-R & ImageNet-S & ID Avg. & OOD Avg. & ID $\Delta$ ($\%$) & OOD $\Delta$ ($\%$)\\
\midrule
Vanilla FT &\textbf{84.20} \color{Gray}(0.02) &75.08 \color{Gray}(0.11)	&26.52 \color{Gray}(0.12)	&46.45 \color{Gray}(0.06)	&48.90 \color{Gray}(0.58) &{79.64}	&40.63 & 0.00 & 0.00\\
LP &77.99 \color{Gray}(0.02)&	67.74 \color{Gray}(0.04)&	27.13 \color{Gray}(0.06)&	50.71 \color{Gray}(0.07)&	46.47 \color{Gray}(0.04)&	72.86&	41.44&	\color{red}-8.51&	\color{Green}2.00\\
BitFit~\cite{zaken2021bitfit} &78.02 \color{Gray}(0.12)&	67.69 \color{Gray}(0.15)&	27.19 \color{Gray}(0.28)&	50.66 \color{Gray}(0.31)&	46.50 \color{Gray}(0.29)&	72.85&	41.45&	\color{red}-8.42&	\color{Green}2.45\\
L2-SP~\cite{xuhong2018explicit} &84.10 \color{Gray}(0.02)&	75.05 \color{Gray}(0.11)&	26.19 \color{Gray}(0.45)&	46.58 \color{Gray}(0.09)&	48.51 \color{Gray}(0.12)&	79.58&	40.43&	\color{red}-0.08&	\color{red}-0.49\\
LP-FT~\cite{kumar2022fine} &83.50 \color{Gray}(0.15)&	73.95 \color{Gray}(0.12)&	25.62 \color{Gray}(0.23)&	46.21 \color{Gray}(0.22)&	48.83 \color{Gray}(0.19)&	78.73&	40.22&	\color{red}-1.15&	\color{red}-1.00\\
Zero-Shot~\cite{radford2021learning} &67.68 \color{Gray}(N/A)	&61.41 \color{Gray}(N/A)	&30.60 (\color{Gray}N/A)	&56.77 \color{Gray}(N/A)	&45.53 \color{Gray}(N/A)	&64.54 	&44.30 	&\color{red}-18.91	&\color{Green}8.64\\
WISE~\cite{wortsman2022robust} &82.11 \color{Gray}(0.14)&	73.61 \color{Gray}(0.13)&	36.11 \color{Gray}(0.16)&	61.77 \color{Gray}(0.08)&	54.16 \color{Gray}(0.07)&	77.86&	50.68&	\color{red}-2.23&	\color{Green}24.75\\
\midrule
TPGM-C &82.41 \color{Gray}(0.07)&	73.91 \color{Gray}(0.21)&	\textbf{36.79} \color{Gray}(0.14)&	\textbf{62.48} \color{Gray}(0.10)&	\textbf{54.91} \color{Gray}(0.12)&	78.16&	\textbf{51.39}&	\color{red}-1.86&	\color{Green}\textbf{26.51}\\

TPGM & 84.19 \color{Gray}(0.03)&	\textbf{75.41} \color{Gray}(1.61)&	34.29 \color{Gray}(2.11)&	57.19 \color{Gray}(0.54)&	54.38 \color{Gray}(0.19)&	\textbf{79.80}&	48.62&	\color{Green}\textbf{0.20}&	\color{Green}19.69\\

\bottomrule
\end{tabular}
}
\label{tab:imagenet}
\end{table*}

%% file: figure/imagenet_analysis.tex
\begin{figure}
     \centering
     \includegraphics[width=0.45\textwidth]{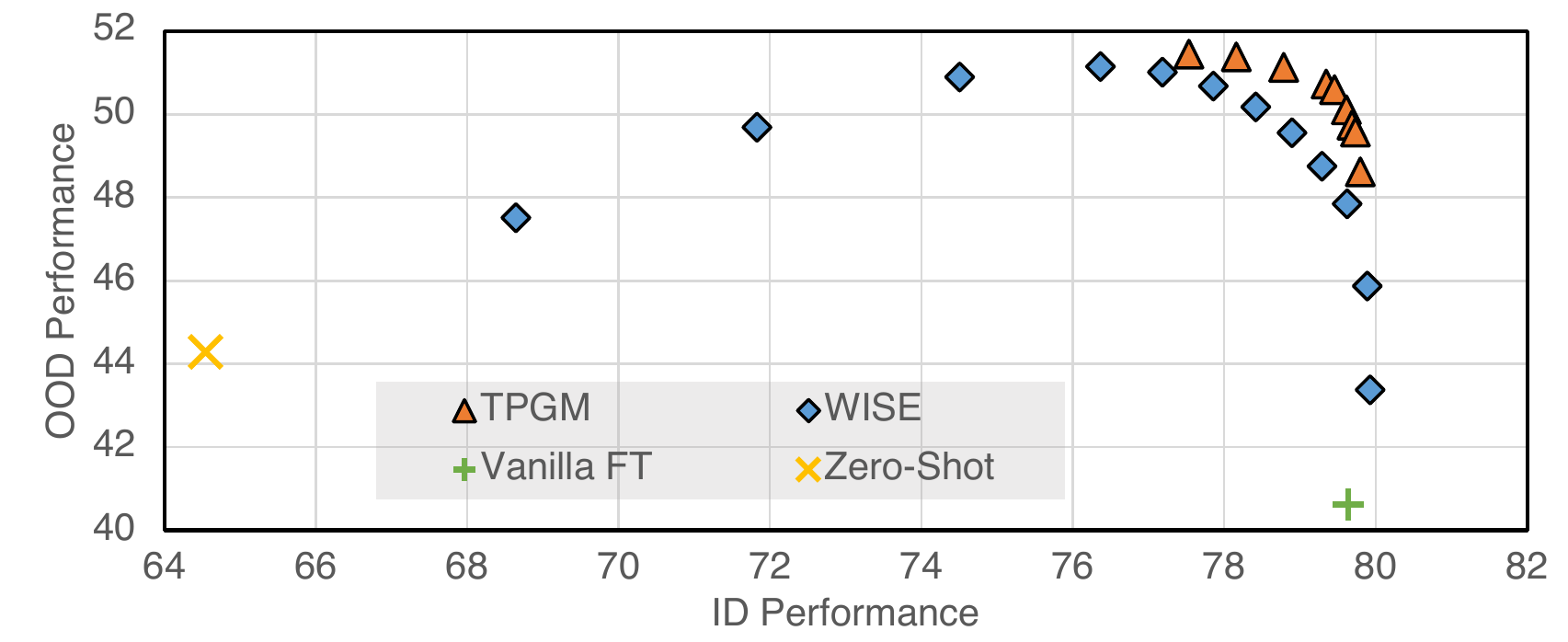}
     \caption{\textbf{ID and OOD performance of TPGM and WISE with different hyperparameters using CLIP pre-trained ViT-B, fine-tuned on ImageNet.} Sweeping different hyperparameters for both WISE and TPGM shows that learning per-layer constraints is superior to learning a single constraint.
   }
     \label{fig:tpgm_wise}
      \vspace{-4mm}
\end{figure}

%% file: figure/imagenet_analysis_b.tex
\begin{figure}
     \centering
     \includegraphics[width=0.45\textwidth]{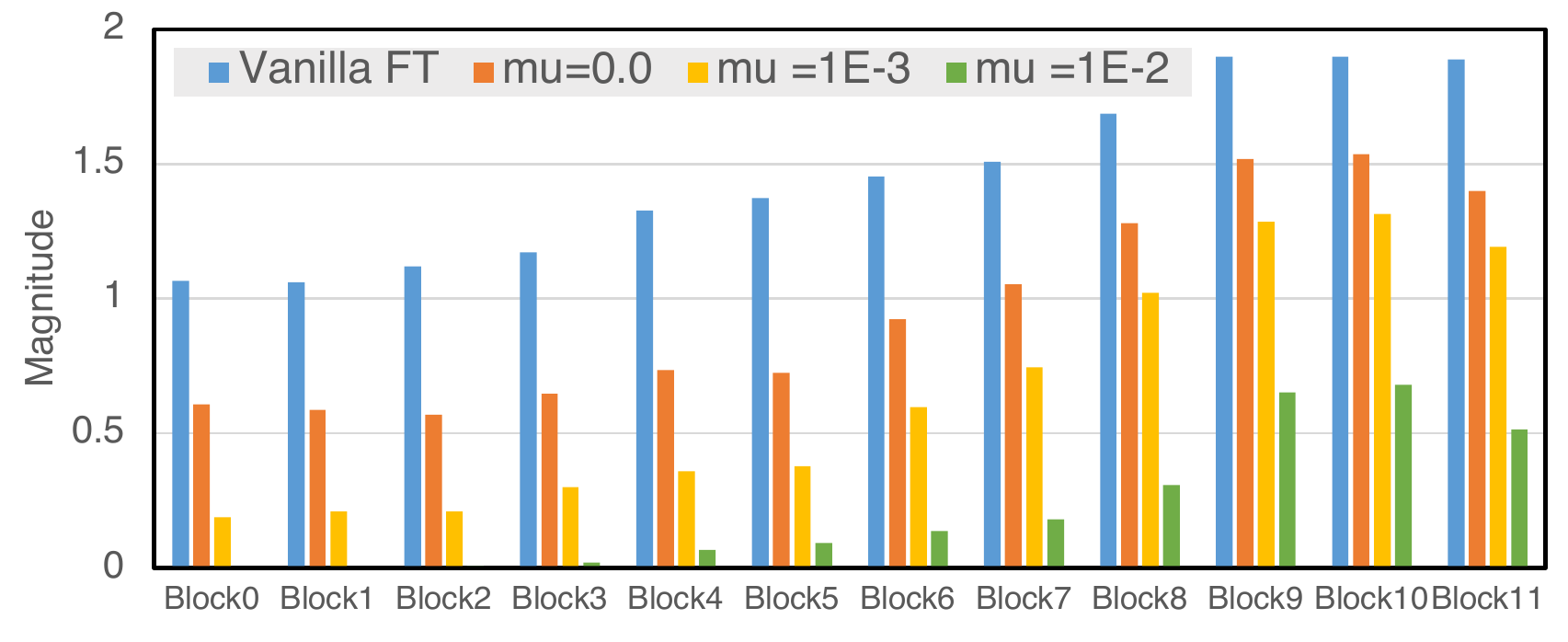}
     \caption{\textbf{Average distance between the fine-tuned model and a CLIP pre-trained ViT-B for each block using TPGM.} Compared to the original distance learned by Vanilla FT, TPGM more aggressively constrains the distance of lower layers.}
     \label{fig:per_layer_para}
     \vspace{-3mm}
\end{figure}

%% file: conclusion.tex
\section{Conclusion}
\looseness=-1  Proposing a bi-level constrained minimization formulation of fine-tuning, we develop the  trainable projected gradient method (TPGM) to learn a distance constraint for each layer of a neural network for robust fine-tuning, which has not been possible with manual hyper-parameter tuning. Our thorough experiments across several pre-trained models and ID/OOD datasets show that TPGM can better preserve the OOD generalization capability of the pre-trained model with minimal effects on ID performance. The optimized constraints exhibit highly interpretable patterns which corroborate existing findings and strengthen the motivation for per-layer constraints.

%% file: appendix.tex
\subsection{Appendix}

\subsection{Proof of Theorem 1}
\label{sec:proof}
We provide complete proof of the main theorem. We will first reiterate the notations used in the main paper. 

\textbf{Problem Setup.}  Let $x\in \mathbb{R}^{d}$ denote an ID data and the corresponding label $y$ is generated by a \textit{ground truth} linear model $\theta_*\in \mathbb{R}^d $, i.e., $y=\theta_*^Tx$.  To construct the training set, we sample $n$ training data, where $n<d$, and stack the sampled data into a data matrix $\mathbf{X}_{tr}\in\mathbb{R}^{d\times n}$. Accordingly, the labels form a vector $\mathbf{Y}_{tr}=\mathbf{X}_{tr}^T\theta_*\in\mathbb{R}^{n}$. The 
\textit{training} goal is to minimize the \textit{empirical} loss.
\begin{align}
    \label{eq:training_obj_ap}
\mathcal{L}(\mathbf{X}_{tr},\mathbf{Y}_{tr};\theta)=\|\mathbf{X}_{tr}^T\theta-\mathbf{Y}_{tr}\|_2
\end{align}

Note that this forms an \textit{over-parameterized} linear system, i.e., there are more parameters than equations, because $n<d$. This is similar to how modern neural networks are over-parameterized with respect to the data.

\textbf{Complementary Decomposition using SVD.} For the analysis, we make an independence assumption on the data matrix $\mathbf{X}_{tr}$. This assumption exists for notation simplicity and can be relaxed easily.
\begin{assumption}
\label{assmp:independence_ap}
Let the $n$ training data be linearly independent. The following SVD exists for the data matrix $\mathbf{X}_{tr}$.
\begin{align*}
\mathbf{X}_{tr} = \mathbf{U}\mathbf{\Sigma} \mathbf{V}^T,\quad \mathbf{U}\in\mathbb{R}^{d\times n},\mathbf{\Sigma}\in\mathbb{R}^{n\times n}, \mathbf{V}\in\mathbb{R}^{n\times n}.
\end{align*}
\end{assumption}

Consequently, we can decompose any vector $x\in\mathbb{R}^{d}$ into two components, $x=\mathbf{U}\tau + \mathbf{U}_{\bot}\tau_{\bot}$, where $\mathbf{U}$ is the basis for the span of training samples, $\mathbf{U}_{\bot}\in\mathbb{R}^{d\times (d-n)}$ is the basis for the complementary subspace, and $\tau\in\mathbb{R}^{n}$, $\tau_{\bot}\in\mathbb{R}^{d-n}$ are the corresponding coordinates. There are infinitely many solutions to Eq.~\ref{eq:training_obj_ap} because this is an over-parameterized system.  The classic result states that,
\begin{align}
\label{eq:least_sqr}
    \theta = \mathbf{U}\mathbf{\Sigma}^{-1}\mathbf{V}^T\mathbf{Y}_{tr}+\mathbf{U}_\bot\beta_\bot,
\end{align}
where $\beta_\bot\in\mathbb{R}^{d-n}$ can be any vector. We denote a \textit{projected model} as $\Tilde{\theta} = \theta_0 + \alpha (\theta - \theta_0)$ (obtained using Eq.~\ref{eq:l2_proj} or Eq.~\ref{eq:mars_proj}), where  $\theta$ is one minimizer of Eq.~\ref{eq:training_obj_ap}, $\theta_0$ is the pre-trained model and $0\leq\alpha\leq1$ is the projection ratio. 

To quantify the effects of projection $\alpha$, we can look at the average performance of the projected model $\Tilde{\theta}$ on test data. Consequently, we investigate the \textit{expected} loss of the projected model over the entire data space. 
\begin{align}
    \label{eq:expected_obj_ap}
    \mathbb{E}[\mathcal{L}(x,y;\Tilde{\theta})]= \mathbb{E}\left[\left\|{\Tilde{\theta}}^Tx-y \right\|_2\right]
\end{align}
We now provide a detailed proof of Theorem~\ref{thm:thm_1} in the main paper.
{\color{black}We first prove two lemmas.
\begin{lemma}
\label{lma:lemma1}
$\|(\theta-\theta_*)^T\mathbf{U}\tau\|_2 = 0$.
\end{lemma}

\begin{proof}
To show it, we use the decomposition in Eq.~\ref{eq:least_sqr}.
\begin{align*}
    \|(\theta-\theta_*)^T\mathbf{U}\tau\|_2& = \| \mathbf{U\Sigma}^{-1}\mathbf{V}^T\mathbf{Y}_{tr}+\mathbf{U}_\bot\beta_\bot-\theta_*)^T\mathbf{U}\tau\|_2\\\nonumber
    & = \|(\mathbf{U\Sigma}^{-1}\mathbf{V}^T\mathbf{Y}_{tr}-\theta_*)^T\mathbf{U}\tau\|_2\\\nonumber
    & = \|(\mathbf{U\Sigma}^{-1}\mathbf{V}^T\mathbf{X}_{tr}^T\theta_*-\theta_*)^T\mathbf{U}\tau\|_2\\\nonumber
    & = \|(\mathbf{U\Sigma}^{-1}\mathbf{V}^T(\mathbf{ U\Sigma V}^T)^T\theta_*-\theta_*)^T\mathbf{U}\tau\|_2\\\nonumber
    &= 0 
\end{align*}
\end{proof}

\begin{lemma}
\label{lma:lemma2}
$\|\mathbf{U}\tau\|_2 \leq \|\tau\|_2$ and  $\|\mathbf{U}_\bot\tau_\bot\|_2 \leq\|\tau_\bot\|_2$.
\end{lemma}
\begin{proof}
    We first invoke the definition of matrix norm, 
    \begin{align*}
        \|\mathbf{U}\|_2 = \sup_{\tau\neq0} \frac{\|\mathbf{U}\tau\|_2 }{\|\tau\|_2}
    \end{align*}

From the definition, it is easy to see that 

\begin{align*}
    \|\mathbf{U}\tau\|_2 \leq \|\mathbf{U}\|_2\|\tau\|_2.
\end{align*}

Now recall that both $\mathbf{U}\in\mathbb{R}^{d\times n}$ and $\mathbf{U}_{\bot}\in\mathbb{R}^{d\times (d-n)}$ are orthonormal matrices. Therefore, using the property of L2 matrix norm,  
\begin{align*}
    \|\mathbf{U}\|_2 = \sqrt{\lambda_{max}(\mathbf{U}^T\mathbf{U})} = \sigma_{max}(\mathbf{U})=1
\end{align*}
where $\lambda_{max}(\cdot)$ and $\sigma_{max}(\cdot)$ denote the largest eigenvalue and singular value respectively. Therefore,
\begin{align*}
    \|\mathbf{U}\tau\|_2 \leq \|\tau\|_2.
\end{align*}
The same analysis extends to $\mathbf{U}_\bot,\tau_\bot$.
\end{proof}

Next, we proceed with the proof of the main theorem.
\begin{proof}
    \begin{align*}
      \mathcal{L}(x,y;\Tilde{\theta}) &= \left\|{\Tilde{\theta}}^Tx-y \right\|_2\\\nonumber
    & = \left\|(\theta_0 + \alpha (\theta - \theta_0))^Tx-\theta_*^Tx \right\|_2\\\nonumber
    & =  \|(\theta_0 + \alpha (\theta - \theta_0)-\theta_*)^T\mathbf{U}\tau +\\\nonumber
    &  (\theta_0 + \alpha (\theta - \theta_0)-\theta_*)^T\mathbf{U}_\bot\tau_\bot \|_2\\\nonumber
    &\leq  \underbrace{\| ((1-\alpha) (\theta_0-\theta_*) + \alpha({\theta}-\theta_*))^T\mathbf{U}\tau\|_2}_{A}+\\\nonumber
    &  \underbrace{\|(\theta_0-\theta_*)^T\mathbf{U}_\bot\tau_\bot\|_2}_{B}  + \underbrace{\| \alpha({\theta}-\theta_0)^T\mathbf{U}_\bot\tau_\bot\|_2}_{C}\\\nonumber
    \end{align*}
We use triangle inequality for the last inequality.
We can now bound $A$ using Lemma~\ref{lma:lemma1}, Cauchy-Schwarz inequality and Lemma~\ref{lma:lemma2} as
\begin{align*}
    \| ((1-\alpha) &(\theta_0-\theta_*) + \alpha({\theta}-\theta_*))^T\mathbf{U}\tau\|_2 \\\nonumber
    & = (1-\alpha)\| (\theta_0-\theta_*)^T\mathbf{U}\tau \|_2\\\nonumber
    & \leq (1-\alpha)\|(\theta_0-\theta_*)\|_2\|\mathbf{U}\tau\|_2\\\nonumber
    & \leq (1-\alpha)\|(\theta_0-\theta_*)\|_2\|\tau\|_2.
\end{align*}

Similarly, we can bound $B$ using Cauchy-Schwarz inequality and Lemma~\ref{lma:lemma2} as 
\begin{align*}
    \|(\theta_0-\theta_*)^T\mathbf{U}_\bot\tau_\bot\|_2 &\leq \|\theta_0-\theta_*\|_2\|\mathbf{U}_\bot\tau_\bot\|_2\\\nonumber
    & \leq \|\theta_0-\theta_*\|_2\|\tau_\bot\|_2,
\end{align*}

and bound $C$ as,
\begin{align*}
    \| \alpha({\theta}-\theta_0)^T\mathbf{U}_\bot\tau_\bot\|_2 &\leq \| \alpha({\theta}-\theta_0)\|_2\mathbf{U}_\bot\tau_\bot\|_2\\\nonumber
    & \leq  \| \alpha({\theta}-\theta_0)\|_2\|\tau_\bot\|_2.
\end{align*}
 Now, plug everything back. We arrive at the final result, 
 \begin{align*}
     \mathcal{L}(x,y;\Tilde{\theta})\leq (1-\alpha)\epsilon\|\tau\|_2 + (\epsilon + \alpha \|\theta-\theta_0\|_2)\|\tau_\bot\|_2
 \end{align*}
 where $\epsilon = \|(\theta_0-\theta_*)\|_2$.
\end{proof}
}

\subsection{Group Based Total Variation Smoothing}
\label{sec:smoothing}
Because of the iterative and incremental nature, the vanilla TPGM algorithm is a \textit{greedy} algorithm, meaning that it judges the \textit{immediate} benefit of the current updates to the model weights. If the current updates are not consistent with the validation data, they will be removed by projection, i.e., the projection radii will not increase to accommodate the new changes. Consequently, projection radii learned by TPGM could be overly \textit{conservative} and lead to underfitting because gradient updates are stochastic, whose benefits may only show up in the long run. Empirically, we found TPGM results in under-fitting in some cases, i.e., slightly lower ID performance. To mitigate this side-effect of iterative optimization, we propose a group-based total variation (TV) smoothing for the projection parameters. TV is a common technique to improve smoothness in image denoising~\cite{chen2010adaptive} and general signal processing~\cite{condat2013direct}.  We propose to utilize TV regularization to enforce a heuristic on the optimization of $\gamma$: \textit{projection ratios of layers in the same group should be similar to each other}. Specifically, modern neural network architectures such as ResNet~\cite{he2016deep} and Transformer~\cite{vaswani2017attention} are modular and stacked with groups (blocks). It is easy to identify unique groups in each architecture and assign layers to each one of them. Therefore, let $\mathcal{G}=\{g_i|i=0,...,M\}$ be the set of unique groups in a neural network. The loss function that we optimize for the projection parameters is updated as the following,
\begin{align}
    \mathcal{L}_\gamma = \mathcal{L}(x,y;\gamma_t) + \mu \sum_{g_i\in\mathcal{G}} \sum_{i\in g_i}|\alpha_i - \alpha_{i-1}|
\end{align}
where $\mu$ is a hyperparameter requiring tuning.

\subsection{Implementation}
\label{sec:implementation}

In Alg.~\ref{alg:alg2}, the \textit{projection update} function has its own optimizer. In our implementation,  we use the Adam~\cite{kingma2014adam} optimizer because of its capability of adapting learning rate. Even though this introduces other hyperparameters, we find the same set of hyperparameters worked well for all experiments. Specifically, we use the default settings and a constant base learning rate of $\zeta=1e-2$.

\textbf{ResNet experiments (Sec.~\ref{sec:resnet_exp}).} We list all the compared methods and their method-specific tuning to reproduce our results.

\begin{itemize}
    \item\textbf{ Vanilla Fine-Tuning (FT)}: We fine-tune all layers and sweep five learning rates (CLIP best $\eta_0=1e-3$, MOCO best $\eta_0=5e-2$).
    \item\textbf{ Linear Probing (LP)}: We only fine-tune the head classifier and sweep five learning rates (CLIP best $\eta_0=1e-1$, MOCO best $\eta_0=1e-1$).
    \item \textbf{Partial Fusion (PF)}~\cite{kanavati2021partial}: We fine-tune all the batch-norm layers and the head classifier, and sweep five learning rates (CLIP best $\eta_0=1e-2$, MOCO best $\eta_0=5e-2$).
    \item \textbf{L2-SP}~\cite{xuhong2018explicit}: We add L2-SP regularization, use the best-validated learning rate from FT, and sweep five three regularization hyperparameters (CLIP best $\mu:1e-2$, MOCO best $\mu:1e-3$).
    \item \textbf{MARS-SP}~\cite{gouk2020distance}: We add MARS projection (Eq~\ref{eq:mars_proj}), use the best-validated learning rate from FT, and sweep five three projection hyperparameters (CLIP best $\mu=64$, MOCO best $\mu=16$).
    \item \textbf{LP-FT}~\cite{kumar2022fine}: We first LP for 25 epochs, using the best LP learning rate, and FT for another 25 epochs, sweeping five learning rates (CLIP best $\eta_0=1e-3$, MOCO best $\eta_0=5e-2$).
    \item \textbf{TPGM}: We learn per-layer L2 projection radii incrementally, sweeping five learning rates (Eq.~\ref{eq:l2_proj}) (CLIP best $\eta_0=1e-2$, MOCO best w/o smoothing $\eta_0=1e-2$, MOCO w/ smoothing best :$\eta_0=4e-2$ and $\mu=0.1$).
\end{itemize}

\textbf{Transformer Experiments (Sec.~\ref{sec:transformer_exp}).} We follow some common practices used in prior works~\cite{touvron2021training,touvron2022deit} to boost fine-tuning performance. Note that we use the same training recipe for all methods unless otherwise specified. For example, linear probing performs worse when augmentations are used~\cite{radford2021learning}. Now we will list the techniques used as well as their corresponding hyperparameters in parenthesis. Specifically, we use label-smoothing ($0.1$)~\cite{szegedy2016rethinking}, weight-decay ($0.1$), Mixup ($0.8$)~\cite{zhang2017mixup} and Cutmix ($1.0$)~\cite{yun2019cutmix}. We fine-tune models using the AdamW optimizer~\cite{loshchilov2017decoupled} for 30 epochs with a warm-up period of 5 epochs~\cite{touvron2021training}, per-step cosine decay schedule~\cite{he2022masked} and a batch size of $512$. We list all the compared methods and their method-specific tuning to reproduce our results. 
\begin{itemize}
    \item\textbf{ Vanilla Fine-Tuning (FT)}: We fine-tune all layers and sweep three learning rate $\eta_0\in\{1e-5,2e-5,3e-5\}$.
    \item\textbf{ Linear Probing (LP)}: We only fine-tune the head classifier and sweep three learning rates $\eta_0\in\{5e-2,1e-2,5e-3\}$. We don't use any data augmentations (e.g., label-smoothing, Mixup and Cutmix) as they decrease LP performance.
    \item \textbf{BitFit}~\cite{zaken2021bitfit}: We fine-tune all the bias terms and the head classifier and sweep three learning rate $\eta_0\in\{5e-2,1e-2,5e-3\}$.
    \item \textbf{L2-SP}~\cite{xuhong2018explicit}: We add L2-SP regularization, use the best-validated learning rate from FT, and sweep three three regularization hyperparameters $\mu\in\{1e-5,1e-4,1e-3\}$.
    \item \textbf{LP-FT}~\cite{kumar2022fine}: We first LP for 15 epochs, sweeping three learning rates $\eta_0\in\{5e-2,1e-2,5e-3\}$, and FT the best-validated model for another 15 epochs, sweeping three learning rate  $\eta_0\in\{1e-5,2e-5,3e-5\}$.
    \item \textbf{Zero-Shot}~\cite{radford2021learning}: We run an inference with the pre-trained CLIP model with the extracted zero-shot classifier.
    \item \textbf{WISE}~\cite{wortsman2022robust}: We linearly interpolate the best validated FT model and the pre-trained model with a ratio of $0.5$.
    \item \textbf{TPGM}: We learn per-layer projection radii between the best validated FT model and the pre-trained model using the MARS projection (Eq.~\ref{eq:mars_proj}). 
\end{itemize}

\subsection{CLIP Pre-trained ViT-L on ImageNet}
\label{sec:vit_l}
In Sec.~\ref{sec:transformer_exp}, we presented fine-tuning results on ImageNet using CLIP pre-trained ViT-b. In this section, we conduct the same experiments with CLIP pre-trained ViT-L. As we noticed in the ViT-b experiments, WISE and TPGM perform much better than other competitors, so we focus on the comparison between the two here. We first present tabulated results in Tab.~\ref{tab:imagenet_vit_l}. We observe that TPGM improves both ID and OOD performance over vanilla FT. To compare fairly with WISE, we introduced TPGM-C (Sec.~\ref{sec:transformer_exp}), which uses an L2 regularization on the learned projection radii to control the distance to the pre-trained model. With proper regularization, TPGM-C outperforms WISE on both ID and OOD performance. We also provide a figure of ID vs. OOD performance with different WISE interpolation ratios and different TPGM-C regularization strengths in Fig.~\ref{fig:tpgm_wise_large}. We observe the same trend as in the ViT-b experiments (Sec.~\ref{sec:transformer_exp}): at the same ID performance, TPGM has better OOD performance.
\input{table/imagenet_vit_l.tex}
\input{figure/imagenet_analysis_l.tex}

\subsection{Comparisons between TPGM-L2 and TPGM-MARS}
\label{sec:compare_proj}
In the main paper, we presented two possible projections: L2 projection (Eq.~\ref{eq:l2_proj}) and MARS projection (Eq.~\ref{eq:mars_proj}). Both projections provide closed-form solutions. We can use either of them in TPGM. In this section, we present comparisons between the two.

\input{table/L2_vs_MARS_resnet}

\textbf{ResNet Experiments on DomainNet.} For ResNet experiment in Sec.~\ref{sec:resnet_exp}, we use a CLIP pre-trained ResNet50 and an ImageNet pre-trained ResNet50. For TPGM, we use $f_{proj}=1$ and $T_{proj} = 1$. In Tab.~\ref{tab:l2_vs_mars_resnet}, we compare the performance of TPGM using MARS and L2 projections on DomainNet-Real with $100\%$ of its data. We observe that in this setting MARS performs better than L2 projection.

\input{figure/l2_vs_mars_transformer}
\textbf{Transformer Experiments on ImageNet.} For Transformer experiments in Sec.~\ref{sec:transformer_exp}, we use a CLIP pre-trained ViT-B. For TPGM, we use $f_{proj}=T-1$ and $T_{proj} = 200$. Following the main paper, we add L2 regularization to the projection radii and sweep a range of values from $4e-3$ to $1e-4$. In Fig.~\ref{fig:l2_vs_mars_transformer}, we compare the performance of TPGM using MARS and L2 projections on ImageNet ID and OOD datasets. We observe that L2 projection always outperforms MARS projection in this setting.

\subsection{WISE for CLIP Pre-trained ResNet}
\label{sec:resnet_wise}
\input{figure/resnet_wise}
The prior work~\cite{wortsman2022robust} and our experiments in Sec.~\ref{sec:transformer_exp} verified that CLIP pre-trained Transformers have very good linear connectivity. This means that when linearly interpolating between the pre-trained model and a fine-tuned model, the output does not degrade much. In this case, we observe significantly improved OOD generalization with minimal ID performance loss. However, the same trend is not observed when switching the architrave to ResNet50. Similar to the prior work~\cite{wortsman2022robust}, we extract a zero-shot classifier for DomainNet classes using a CLIP pre-trained ResNet50 and conduct the same linear interpolation ratio sweeping as in the main paper. In Fig.~\ref{fig:resnet_wise}, we plot ID performance against the OOD performance of WISE with different ratios, the pre-trained (zero-shot) model, and the vanilla fine-tuned model. Notably, we observe a significant drop in performance when interpolating between the pre-trained model and a fine-tuned model. This shows that CLIP pre-trained ResNet does not enjoy the same linear connectivity as its Transformer counterpart.

\subsection{Smoothing Comparison using MOCO-V3}
\label{sec:moco_smoothing}
\input{table/moco_tv_smoothing}
In the main paper, we found that training with MOCO-V3 pre-trained ResNet50 on DomainNet can benefit from total variation (TV) smoothing (Appendix~\ref{sec:smoothing}). Here we show a detailed comparison between TPGM with and without smoothing for this particular setting in Tab.~\ref{tab:moco_smoothing}. We observe that TPGM without smoothing achieves the best OOD performance however with a slight decrease in ID performance compared to vanilla FT. This might be caused by the conservative nature of TPGM as discussed in Appendix~\ref{sec:smoothing}. When TV smoothing is added, we observe that TPGM brings improvement to both ID and OOD performance over vanilla FT.

\subsection{Computation Overhead}
\label{sec:computation}
TPGM inevitably adds some computation overhead to a vanilla fine-tuning pipeline (though not inference). The majority of computation cost comes in Alg.~\ref{alg:alg2}, where the algorithm needs to conduct gradient updates on the projection parameters. While this overhead is negligible when we set $f_{freq} = T-1$, i.e., \textit{projection update} is only called once at the end of training as in the Transformer experiments (Sec.~\ref{sec:transformer_exp}), the overhead increases when $f_{freq} = 1$. In our ResNet experiments (Sec.~\ref{sec:resnet_exp}), to decrease computation cost, we only the update projection parameters once during each call, i.e, $T_{proj}=1$. Qualitatively, we see an increase of training time from $\sim29$ hours to $\sim34$ hours when TPGM is added, a $\sim17\%$ increase. However, this increase can be justified by the fact that manually searching for per-layer constraints can be intractable.

%% file: table/imagenet_vit_l.tex
\begin{table*}[h!]
\caption{\textbf{ImageNet Results using CLIP pre-trained ViT-L.} TPGM improves OOD performance significantly without losing ID performance. TPGM-C achieves the best OOD performance while maintaining a more competitive ID performance compared to the current state-of-the-art method WISE. TPGM-C is a controlled variant of TPGM, designed to lower its ID performance to the same level as WISE for a fair comparison of OOD performance.
}
\centering

\resizebox{0.9\linewidth}{!}{
\begin{tabular}{c|cc|ccc|cccc} 

\toprule
&\multicolumn{2}{c|}{ID} & \multicolumn{3}{c|}{OOD} & \multicolumn{4}{c}{Statistics} \\

&ImageNet & ImageNet-V2 & ImageNet-A & ImageNet-R & ImageNet-S & ID Avg. & OOD Avg. & ID $\Delta$ ($\%$) & OOD $\Delta$ ($\%$)\\
\midrule
Vanilla FT &\textbf{87.24}	&79.25	&49.67&	63.29&	61.62&	83.25&	58.19&	0.00&	0.00\\
Zero-Shot~\cite{radford2021learning} &75.00&	69.95&	52.21&	71.69&	58.24&	72.48&	60.71&	\color{red}-12.94&	\color{Green}4.33\\
WISE~\cite{wortsman2022robust} &85.33&	78.50&	58.26&	75.37&	64.84&	81.92&	66.16&	\color{red}-1.60&	\color{Green}13.68\\
\midrule
TPGM-C &86.02&	78.83&	\textbf{59.29}&	\textbf{76.32}&	65.00&	82.43&	\textbf{66.87}&	\color{red}-0.99&	\color{Green}\textbf{14.91}\\
TPGM &87.00&	\textbf{79.81}&	58.31&	74.41&	\textbf{65.13}&	\textbf{83.41}&	65.95&	\color{Green}\textbf{0.19}&	\color{Green}13.33 \\

\bottomrule
\end{tabular}
}
\label{tab:imagenet_vit_l}
\end{table*}

%% file: figure/imagenet_analysis_l.tex
\begin{figure}
     \centering
     \includegraphics[width=0.45\textwidth]{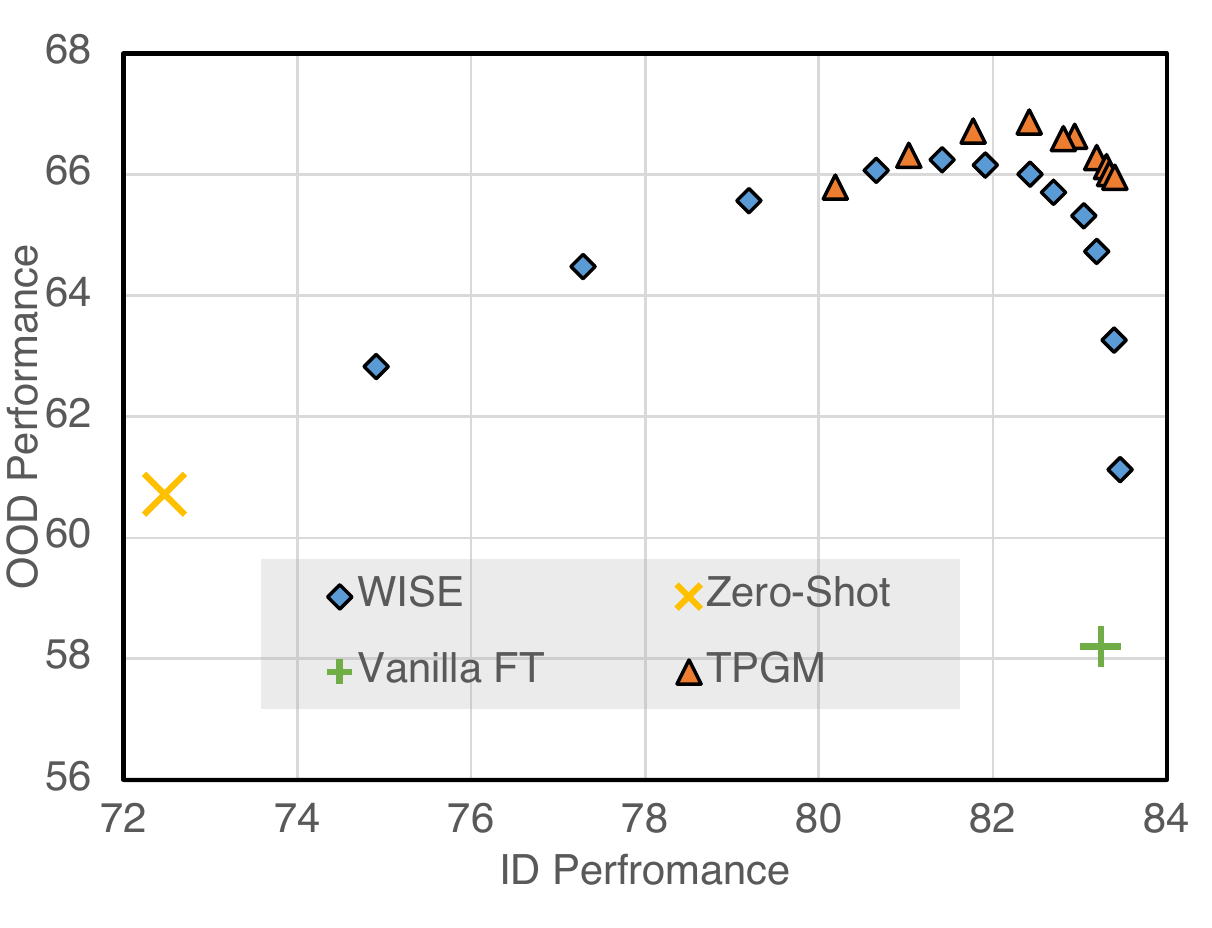}
     \caption{\textbf{ID and OOD performance of TPGM and WISE with different hyperparameters using CLIP pre-trained ViT-L, fine-tuned on ImageNet.} Sweeping different hyperparameters for both WISE and TPGM shows that learning per-layer constraints is superior to learning a single constraint.
   }
     \label{fig:tpgm_wise_large}
\end{figure}

%% file: table/L2_vs_MARS_resnet.tex
\begin{table}[t]
\caption{\textbf{Comparison between MARS and L2 projections on DomainNet using ResNet50}.
}
\centering

\resizebox{0.9\linewidth}{!}{
\begin{tabular}{c|c|c|cccc|c} 

\toprule
&&\multicolumn{1}{c|}{ID} & \multicolumn{4}{c|}{OOD} & \multicolumn{1}{c}{Statistics} \\

&&Real & Sketch & Painting & Infograph & Clipart &  OOD Avg. \\
\midrule

\multirow{2}{*}{CLIP}
&MARS&\textbf{83.64} &	\textbf{38.78} &	43.11 &	\textbf{28.70} &	\textbf{48.01} &	\textbf{39.65}\\
& L2 &82.72&	37.18&	\textbf{43.33}&	25.99&	45.71&	38.05\\
\midrule
\multirow{2}{*}{MOCO}
& MARS & 81.66&	\textbf{35.97}&	\textbf{46.68} &	\textbf{20.34}&	\textbf{46.11} &	\textbf{37.27}\\
 & L2 & 81.66	&33.96	&45.82	&18.71	&44.45	&35.74\\
\bottomrule
\end{tabular}
}
\label{tab:l2_vs_mars_resnet}
\end{table}

%% file: figure/l2_vs_mars_transformer.tex
\begin{figure}
     \centering
     \includegraphics[width=0.45\textwidth]{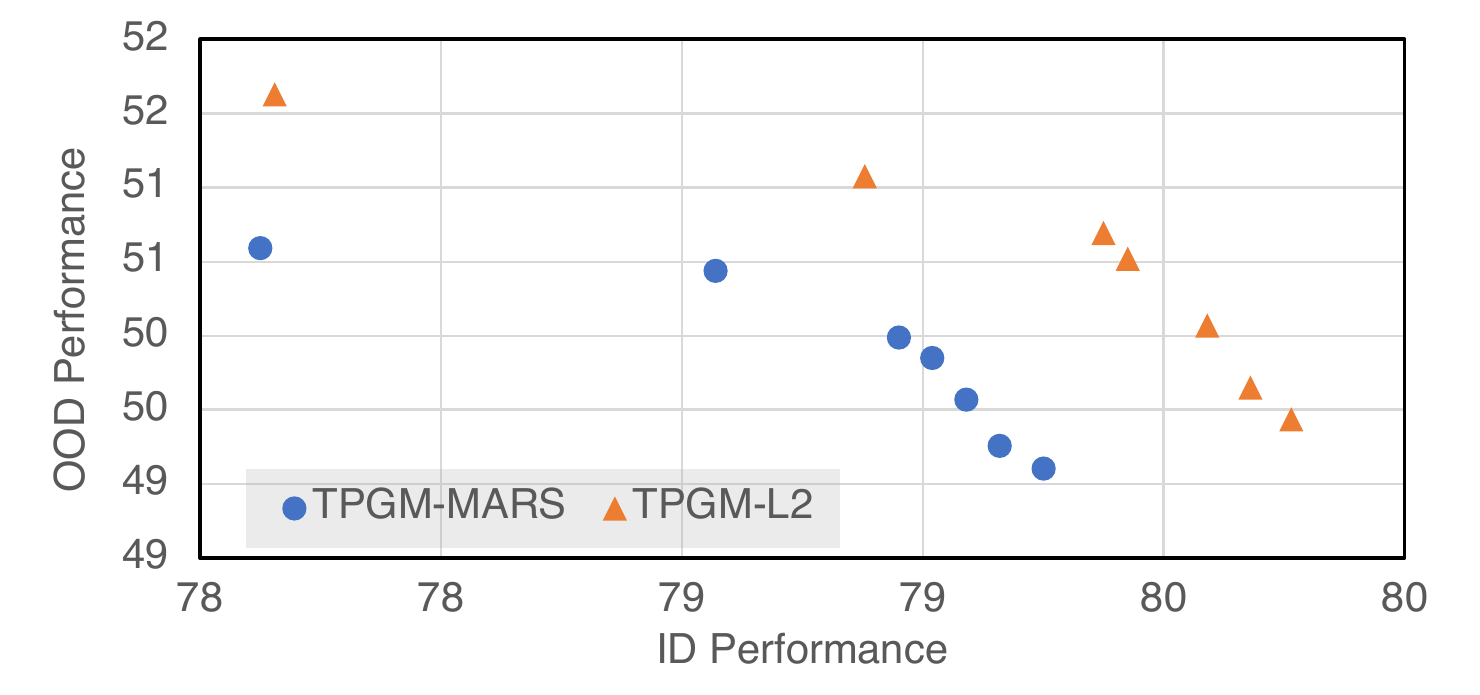}
     \caption{\textbf{Comparison between MARS and L2 projections on ImageNet using ViT-B.}}
     \label{fig:l2_vs_mars_transformer}
\end{figure}

%% file: figure/resnet_wise.tex
\begin{figure}
     \centering
     \includegraphics[width=0.45\textwidth]{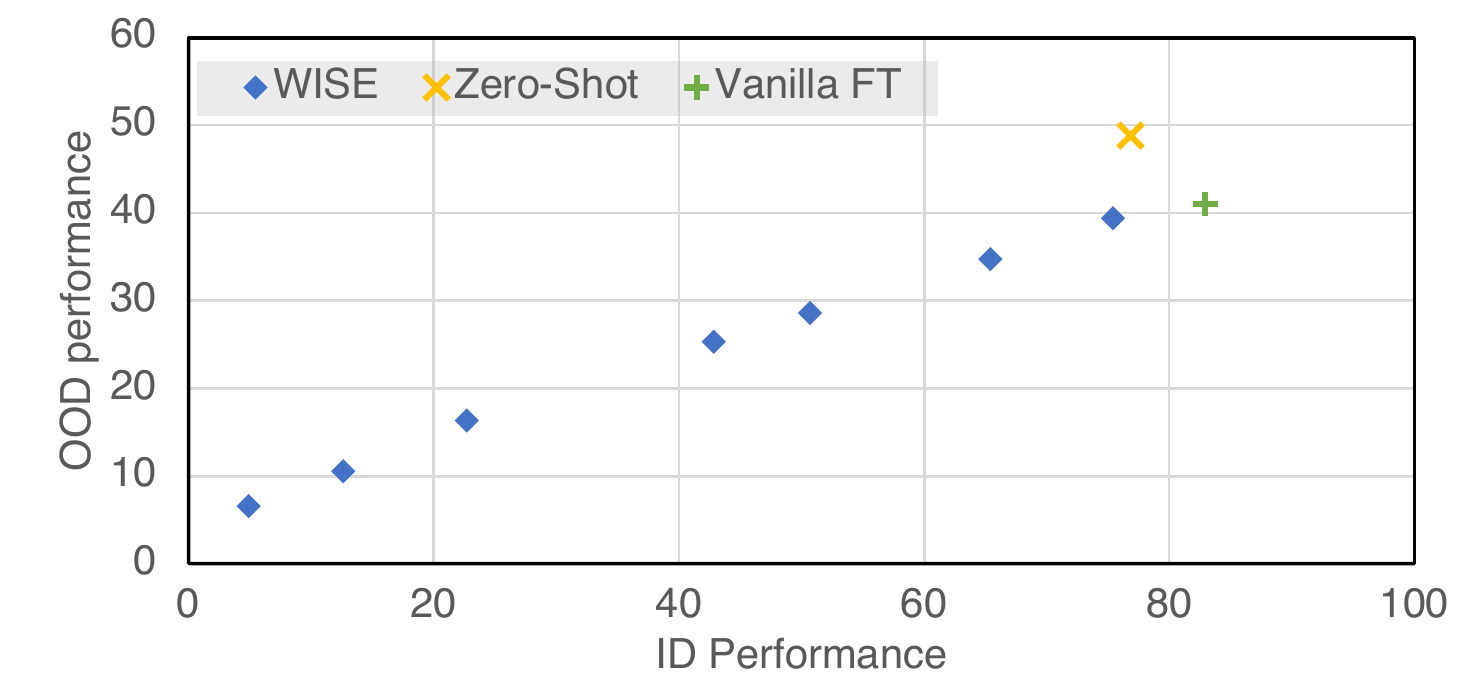}
     \caption{\textbf{WISE interpolation ratio sweeping using CLIP pre-trained ResNet50 on DomainNet.}
   }
     \label{fig:resnet_wise}
\end{figure}

%% file: table/moco_tv_smoothing.tex
\begin{table}[t]
\caption{\textbf{DomainNet Results using MOCO-V3 pre-trained ResNet50 Results with 100$\%$ Real Data.} TPGM without TV smoothing achieves the best OOD performance but with slightly worse ID performance compared to vanilla FT. TV smoothing can effectively mitigate this negative effect. 
}
\centering

\resizebox{0.9\linewidth}{!}{
\begin{tabular}{c|c|cccc|c} 

\toprule
&\multicolumn{1}{c|}{ID} & \multicolumn{4}{c|}{OOD} & \multicolumn{1}{c}{Statistics} \\

&Real & Sketch & Painting & Infograph & Clipart &  OOD Avg\\
\midrule
Vanilla FT & 81.99 &	31.52 &	42.89 &	18.51 &	44.98 &	34.47 \\
\midrule
TPGM w/o TV& 81.66 &	\textbf{35.97} &	\textbf{46.68} &	\textbf{20.34} &	46.11 &	\textbf{37.27}\\
TPGM w/ TV& \textbf{82.66} &	{35.35} &	{46.20}&	{20.13 }&	45.75 &	{36.86}\\
\bottomrule
\end{tabular}
}
\label{tab:moco_smoothing}
\end{table}